\documentclass[dvipdfmx]{article} 
\usepackage{nips15submit_e,times}
\usepackage{url}
\usepackage[dvipdfmx]{graphicx} 
\usepackage{algorithm}
\usepackage{algorithmic}
\usepackage{bm}
\usepackage{amsfonts}
\usepackage{amsthm}
\usepackage{amsmath}
\usepackage{amssymb} 
\usepackage{wrapfig}

\usepackage{subfigure}
\newlength{\subfigwidth}
\newlength{\subfigcolsep}
\title{Regret Lower Bound and Optimal Algorithm in Finite Stochastic Partial Monitoring}

\author{
Junpei Komiyama\\
The University of Tokyo\\
\texttt{junpei@komiyama.info} \\
\And
Junya Honda \\
The University of Tokyo\\
\texttt{honda@stat.t.u-tokyo.ac.jp} \\
\And
Hiroshi Nakagawa \\
The University of Tokyo\\
\texttt{nakagawa@dl.itc.u-tokyo.ac.jp}
}

\DeclareMathOperator*{\argmin}{\mathop{\rm arg~min}}
\newtheorem{theorem}{Theorem}

\newtheorem{corollary}[theorem]{Corollary}
\newtheorem{lemma}[theorem]{Lemma}

\newcommand{\Real}{\mathbb{R}}
\newcommand{\Prob}{\Pr}
\newcommand{\Probp}{\mathbb{P}'}

\newcommand{\Expect}{\mathbb{E}}
\newcommand{\Expectp}{\mathbb{E}'}
\newcommand{\Ind}{\id}

\newcommand{\simplex}{\mathcal{P}_M}

\newcommand{\ist}{{i^*}}
\newcommand{\imd}{{i'}}
\newcommand{\pst}{p^*}
\newcommand{\pmd}{p'}

\newcommand{\Regret}[1]{\mathrm{Regret}(#1)}

\newcommand{\actions}{[N]}
\newcommand{\outcomes}{[M]}
\newcommand{\signali}[1]{S_{#1}}
\newcommand{\signal}[3]{(S_{#1})_{#2,#3}}
\newcommand{\nib}[1]{N_{#1}(T)}

\newcommand{\cell}[1]{\mathcal{C}_{#1}}

\newcommand{\nn}{\nonumber\\}

\newcommand{\Img}{\mathrm{Im}}
\newcommand{\Obsi}{{\hat{X}_{i}}}
\newcommand{\hKL}{\mathrm{\widehat{KL}}}
\newcommand{\Natural}{\mathbb{N}}

\newcommand{\hatmu}{\hat{\mu}}

\newcommand{\tTheta}{\tilde{\Theta}}
\newcommand{\nonoptimal}{{\cP_{1}^c}}
\newcommand{\nonoptimali}[1]{{\cP_{#1}^c}}
\newcommand{\mymod}{\text{ mod }}

\newcommand{\ED}{\mathcal{D}}
\newcommand{\EE}{\mathcal{E}}

\newcommand{\admq}[2]{\mathcal{R}_{#1}(#2)}

\nipsfinalcopy 

\newcommand{\so}{\mathrm{o}}
\newcommand{\lo}{\mathrm{O}}
\newcommand{\ep}{\epsilon}
\newcommand{\de}{\delta}
\newcommand{\al}{\alpha}

\newcommand{\id}{ \mbox{\rm 1}\hspace{-0.63em}\mbox{\rm \small 1\,}}
\newcommand{\idx}[1]{\id\left[#1\right]}
\newcommand{\E}{\Expect}

\newcommand{\n}{\nonumber}
\newcommand{\rd}{\mathrm{d}}
\newcommand{\sets}{\mathcal{H}}

\newcommand{\interior}{\mathrm{int}}
\newcommand{\closure}{\mathrm{cl}}

\newcommand{\e}{\mathrm{e}}
\newcommand{\De}{\Delta}

\newcommand{\com}{\,,}
\newcommand{\per}{\,.}

\newcommand{\since}[1]{\quad\left(\mbox{#1}\right)}

\newcommand{\ihat}{\hat{i}}
\newcommand{\cA}{\mathcal{A}}

\newcommand{\cP}{\mathcal{C}}

\newcommand{\variation}{\mathcal{B}}
\newcommand{\trueD}{\mathcal{C}}
\newcommand{\empD}{\mathcal{D}}
\newcommand{\other}{\mathcal{E}}

\newcommand{\phat}{\hat{p}}
\newcommand{\adm}[2]{\mathcal{R}_{#1}(#2)}
\newcommand{\admb}[2]{\bar{\mathcal{R}}_{#1}(#2)}

\newcommand{\optset}[2]{\mathcal{R}_{#1}^*(#2)}

\newcommand{\optc}[2]{C_{#1}^*(#2)}
\newcommand{\optcone}[1]{C_{#1}^*}
\newcommand{\jt}[2]{J_{#1}(#2)}
\newcommand{\jtp}[2]{J'_{#1}(#2)}
\newcommand{\kyoria}[1]{\Vert#1\Vert}
\newcommand{\kyorim}[1]{\Vert#1\Vert_M}
\newcommand{\norm}[1]{}
\newcommand{\normss}{}
\newcommand{\normsd}{}
\newcommand{\regret}{\mathrm{Regret}}
\newcommand{\scup}{\,\cup\,}
\newcommand{\setd}{\mathcal{S}}
\newcommand{\tp}{\tilde{p}}
\newcommand{\Lmax}{L_{\mathrm{max}}}

\begin{document}

\maketitle

\begin{abstract}
Partial monitoring is a general model for sequential learning with limited feedback formalized as a game between two players.
In this game, the learner chooses an action and at the same time the opponent chooses an outcome, then the learner suffers a loss and receives a feedback signal. The goal of the learner is to minimize the total loss. 
In this paper, we study partial monitoring with finite actions and stochastic outcomes.
We derive a logarithmic distribution-dependent regret lower bound that defines the hardness of the problem.
Inspired by the DMED algorithm (Honda and Takemura, 2010) for the multi-armed bandit problem, we propose PM-DMED, an algorithm that minimizes the distribution-dependent regret. 
PM-DMED significantly outperforms state-of-the-art algorithms in numerical experiments.
To show the optimality of PM-DMED with respect to the regret bound, we slightly modify the algorithm by introducing a hinge function (PM-DMED-Hinge).
Then, we derive an asymptotically optimal regret upper bound of PM-DMED-Hinge that matches the lower bound.
\end{abstract}

\section{Introduction}

Partial monitoring is a general framework for sequential decision making problems with imperfect feedback. Many classes of problems, including prediction with expert advice \cite{Littlestone:1994:WMA:184036.184040}, the multi-armed bandit problem \cite{LaiRobbins1985}, dynamic pricing \cite{DBLP:conf/focs/KleinbergL03}, the dark pool problem \cite{DBLP:journals/jmlr/AgarwalBD10}, label efficient prediction \cite{DBLP:journals/tit/Cesa-BianchiLS05}, and linear and convex optimization with full or bandit feedback \cite{DBLP:conf/icml/Zinkevich03,DBLP:conf/colt/DaniHK08} can be modeled as an instance of partial monitoring. 

Partial monitoring is formalized as a repeated game played by two players called a learner and an opponent.
At each round, the learner chooses an action, and at the same time the opponent chooses an outcome. Then, the learner observes a feedback signal from a given set of symbols and suffers some loss, both of which are deterministic functions of the selected action and outcome. 

The goal of the learner is to find the optimal action that minimizes his/her cumulative loss.
Alternatively, we can define the regret as the difference between the cumulative losses of the learner and the single optimal action, and minimization of the loss is equivalent to minimization of the regret.
A learner with a small regret balances exploration (acquisition of information about the strategy of the opponent) and exploitation (utilization of information). The rate of regret indicates how fast the learner adapts to the problem: a linear regret indicates the inability of the learner to find the optimal action, whereas a sublinear regret indicates that the learner can approach the optimal action given sufficiently large time steps.

The study of partial monitoring is classified into two settings with respect to the assumption on the outcomes. 
On one hand, in the stochastic setting, the opponent chooses an outcome distribution before the game starts, and an outcome at each round is an i.i.d.\,sample from the distribution. 
On the other hand, in the adversarial setting, the opponent chooses the outcomes to maximize the regret of the learner.
In this paper, we study the former setting.

\subsection{Related work}

The paper by Piccolboni and Schindelhauer \cite{PiccolboniChristian2001} is one of the first to study the regret of the finite partial monitoring problem.
They proposed the FeedExp3 algorithm, which attains $\lo(T^{3/4})$ minimax regret on some problems. This bound was later improved by Cesa-Bianchi et al.\,\cite{DBLP:journals/mor/Cesa-BianchiLS06} to $\lo(T^{2/3})$, who also showed an instance in which the bound is optimal.
Since then, most literature on partial monitoring has dealt with the minimax regret, which is the worst-case regret over all possible opponent's strategies.
Bart{\'{o}}k et al.\,\cite{DBLP:journals/jmlr/BartokPS11} classified the partial monitoring problems into four categories in terms of the minimax regret:
 a trivial problem with zero regret, an easy problem with $\tTheta(\sqrt{T})$ regret\footnote{Note that $\tTheta$ ignores a polylog factor.}, a hard problem with $\Theta(T^{2/3})$ regret, and a hopeless problem with $\Theta(T)$ regret. 
This shows that the class of the partial monitoring problems is not limited to the bandit sort but also includes larger classes of problems, such as dynamic pricing.
Since then, several algorithms with a $\tilde{O}(\sqrt{T})$ regret bound for easy problems have been proposed \cite{bartokicml12,bartokcolt13,efficientpm}. Among them, the Bayes-update Partial Monitoring (BPM) algorithm \cite{efficientpm} is state-of-the-art in the sense of empirical performance.

\textbf{Distribution-dependent and minimax regret:}
we focus on the distribution-dependent regret that depends on the strategy of the opponent.
While the minimax regret in partial monitoring has been extensively studied,
 little has been known on distribution-dependent regret in partial monitoring.
To the authors' knowledge, the only paper focusing on the distribution-dependent regret in finite discrete partial monitoring is the one by Bart{\'{o}}k et al.\,\cite{bartokicml12}, which derived $\lo(\log{T})$ distribution-dependent regret for easy problems.
In contrast to this situation, much more interest in the distribution-dependent regret has been shown in the field of multi-armed bandit problems. 
Upper confidence bound (UCB), the most well-known algorithm for the multi-armed bandits, has a distribution-dependent regret bound \cite{LaiRobbins1985,auerfinite}, and algorithms that minimize the distribution-dependent regret (e.g., KL-UCB) has been shown to perform better than ones that minimize the minimax regret (e.g., MOSS), even in instances in which the distributions are hard to distinguish (e.g., Scenario 2 in Garivier et al.\,\cite{GarivierKLUCB}).
Therefore, in the field of partial monitoring,
 we can expect that an algorithm that minimizes the distribution-dependent regret would perform better than the existing ones. 
 
\noindent\textbf{Contribution:}
the contributions of this paper lie in the following three aspects.
First, we derive the regret lower bound: in some special classes of partial monitoring (e.g., multi-armed bandits), an $\lo(\log{T})$ regret lower bound is known to be achievable. In this paper, we further extend this lower bound to obtain a regret lower bound for general partial monitoring problems.
Second, we propose an algorithm called Partial Monitoring DMED (PM-DMED). We also introduce a slightly modified version of this algorithm (PM-DMED-Hinge) and derive its regret bound. PM-DMED-Hinge is the first algorithm with a logarithmic regret bound for hard problems. Moreover, for both easy and hard problems, it is the first algorithm with the optimal constant factor on the leading logarithmic term.
Third, performances of PM-DMED and existing algorithms are compared in numerical experiments. Here, the partial monitoring problems consisted of three specific instances of varying difficulty. In all instances, PM-DMED significantly outperformed the existing methods when a number of rounds is large.
The regret of PM-DMED on these problems quickly approached the theoretical lower bound.

\section{Problem Setup}

This paper studies the finite stochastic partial monitoring problem with $N$ actions, $M$ outcomes, and $A$ symbols.
An instance of the partial monitoring game is defined by a loss matrix $L = (l_{i,j}) \in \Real^{N \times M}$ and a feedback matrix $H = (h_{i,j}) \in [A]^{N \times M}$, where $[A] = \{1,2,\dots,A\}$. At the beginning, the learner is informed of $L$ and $H$.
At each round $t = 1,2,\dots,T$, a learner selects an action $i(t) \in \actions$, and at the same time an opponent selects an outcome $j(t) \in \outcomes$. The learner suffers loss $l_{i(t), j(t)}$, which he/she cannot observe: the only information the learner receives is the signal $h_{i(t),j(t)} \in [A]$. 
We consider a stochastic opponent whose strategy for selecting outcomes is governed by the opponent's strategy $\pst \in \simplex$, where $\simplex$ is a set of probability distributions over an $M$-ary outcome.
The outcome $j(t)$ of each round is an i.i.d.\,sample from $\pst$.

\begin{wrapfigure}{r}{4cm} 
\vspace{-1em}
\includegraphics[scale=0.2]{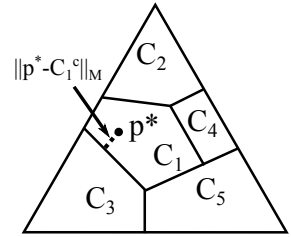}
\caption{Cell decomposition of a partial monitoring instance with $M=3$.}
\label{fig:celldecomp}
\vspace{-2em}
\end{wrapfigure} 

The goal of the learner is to minimize the cumulative loss over $T$ rounds.
Let the optimal action be the one that minimizes the loss in expectation, that is, $\ist = \argmin_{i \in \actions} L_i^\top \pst$, where $L_i$ is the $i$-th row of $L$. Assume that $\ist$ is unique. Without loss of generality, we can assume that $\ist = 1$.
Let $\De_i = (L_i - L_1)^\top \pst \in [0, \infty)$ and $N_i(t)$ be the number of rounds before the $t$-th in which action $i$ is selected.
The performance of the algorithm is measured by the (pseudo) regret,
\begin{equation*}
 \Regret{T} = \sum_{t=1}^T \De_{i(t)} = \sum_{i \in [N]} \De_{i} N_i(T+1),
\end{equation*} 
which is the difference between the expected loss of the learner and the optimal action $1$.
It is easy to see that minimizing the loss is equivalent to minimizing the regret. The expectation of the regret measures the performance of an algorithm that the learner uses.

For each action $i \in \actions$, let $\cell{i}$ be the set of opponent strategies for which action $i$ is optimal:
\begin{equation*}
 \cell{i} = \{q \in \simplex: \forall_{j \neq i} (L_i - L_j)^\top q \leq 0 \}.
\end{equation*}
We call $\cell{i}$ the optimality cell of action $i$. 
Each optimality cell is a convex closed polytope.
Furthermore, we call the set of optimality cells $\{\cell{1},\dots,\cell{N}\}$ the cell decomposition as shown in Figure \ref{fig:celldecomp}. Let $\nonoptimali{i} = \simplex \setminus \cell{i}$ be the set of strategies with which action $i$ is not optimal. 

The signal matrix $\signali{i} \in \{0,1\}^{A \times M}$ of action $i$ is defined as $\signal{i}{k}{j} = \Ind[h_{i,j} = k]$, where $\Ind[X] = 1$ if $X$ is true and $0$ otherwise.
The signal matrix defined here is slightly different from the one in the previous papers (e.g., Bart{\'{o}}k et al.\,\cite{DBLP:journals/jmlr/BartokPS11}) in which the number of rows of $\signali{i}$ is the number of the different symbols in the $i$-th row of $H$. The advantage in using the definition here is that, $\signali{i} \pst \in \Real^A$ is a probability distribution over symbols that
 the algorithm observes when it selects an action $i$. Examples of signal matrices are shown in Section \ref{sec:experiment}.
An instance of partial monitoring is {\it globally observable} if for all pairs $i,j$ of actions, $L_i - L_j \in \oplus_{k \in \actions} \Img S_k^\top$.
In this paper, we exclusively deal with globally observable instances: in view of the minimax regret, this includes trivial, easy, and hard problems. 

\section{Regret Lower Bound}

A good algorithm should work well against any opponent's strategy. 
We extend this idea by introducing the notion of strong consistency:
 a partial monitoring algorithm is strongly consistent
if it satisfies $\Expect[\Regret{T}] = \so(T^a)$ for any $a>0$ and $p \in \simplex$ given $L$ and $H$.

In the context of the multi-armed bandit problem, Lai and Robbins \cite{LaiRobbins1985} derived the regret lower bound of a strongly consistent algorithm: an algorithm must select each arm $i$ until its number of draws $N_i(t)$ satisfies $\log{t} \lesssim N_i(t) d(\theta_i \Vert \theta_1)$, where $d(\theta_i \Vert \theta_1)$ is the KL divergence between the two one-parameter distributions from which the rewards of action $i$ and the optimal action are generated. 
Analogously, in the partial monitoring problem, we can define the minimum number of observations.
\begin{lemma} 
For sufficiently large $T$, a strongly consistent algorithm satisfies:
\begin{equation*}
 \forall_{q \in \nonoptimal} \sum_{i \in \actions} \Expect[N_i(T)] D(\pst_i \Vert S_i q) \geq \log{T} - \so(\log{T}), 
\end{equation*}
where $\pst_i = S_i \pst$ and $D(p \Vert q) = \sum_{i} (p)_i \log{((p)_i/(q)_i)}$ is the KL divergence between two discrete distributions, in which we define $0 \log{0/0} = 0$.
\label{lem:drawlower}
\vspace{-0.3em}
\end{lemma}%
Lemma \ref{lem:drawlower} can be interpreted as follows: for each round $t$, consistency requires the algorithm to make sure that the possible risk that  action $i \neq 1$ is optimal is smaller than $1/t$.
Large deviation principle \cite{dembozeitouni} states that, the probability that an opponent with strategy $q$ behaves like $\pst$ is roughly
$\exp{(-\sum_i N_i(t) D(\pst_i \Vert S_i q))}$. Therefore, we need to continue exploration of the actions until $\sum_i N_i(t) D(\pst_i \Vert S_i q) \sim \log{t}$ holds for any $q \in \nonoptimal$ to reduce the risk to $\exp{(-\log{t})} = 1/t$.

The proof of Lemma \ref{lem:drawlower} is in Appendix \ref{sec:regretlowerproof} in the supplementary material. 
Based on the technique used in Lai and Robbins \cite{LaiRobbins1985}, the proof considers a modified game in which another action $i \neq 1$ is optimal.
The difficulty in proving the lower bound in partial monitoring lies in that, the feedback structure can be quite complex: for example, to confirm the superiority of action $1$ over $2$, one might need to use the feedback from action $3 \notin \{1,2\}$.
Still, we can derive the lower bound by utilizing the consistency of the algorithm in the original and modified games.

We next derive a lower bound on the regret based on Lemma \ref{lem:drawlower}. 
Note that, the expectation of the regret can be expressed as $\Expect[\Regret{T}] = \sum_{i \neq 1} \Expect[N_i(t)] (L_i - L_1)^\top \pst$.
Let
\begin{align*}
\admq{j}{\{p_i\}}
&=
\biggl\{
\{r_i\}_{i\neq j}\in [0,\infty)^{N-1}: 
\inf_{q\in \closure(\cP_{j}^c): p_j = S_j q} \sum_{i} r_i D(p_i\Vert S_i q) \ge 1
\biggr\},
\end{align*}
where $\closure(\cdot)$ denotes a closure. Moreover, let
\begin{equation*}
 \optcone{j}(p,\{p_i\}) = \inf_{r_i\in\admq{j}{\{p_i\}}} \sum_{i \neq j} r_i (L_i - L_j)^\top p \com
\end{equation*}
the optimal solution of which is
\begin{align*}
\optset{j}{p,\{p_i\}}
=\biggl\{ \hspace{-0.3em} \{r_i\}_{i\neq j}\in \admq{j}{\{ p_i \}}:
\hspace{-0.1em} \sum_{i\neq j} r_i (L_i-L_j)^\top p
=
\optcone{j}(p,\{p_i\})
\biggr\}\per
\end{align*}

The value $\optcone{1}(p^*,\{p^*_i\}) \log{T}$ is the possible minimum regret for observations such that the minimum divergence of $\pst$ from any $q \in \nonoptimal$ is larger than $\log{T}$.
Using Lemma \ref{lem:drawlower} yields the following regret lower bound:
\begin{theorem} 
The regret of a strongly consistent algorithm is lower bounded as:
\begin{equation*}
\E[\regret(T)]\ge
\optcone{1}(p^*,\{p_i^*\}) \log T\n
-\so(\log T).
\end{equation*}
\label{thm:regretlower}
\vspace{-1.5em}
\end{theorem}%
From this theorem, we can naturally
measure the harshness of the instance by $\optcone{1}(p^*,\{p_i^*\})$,
whereas the past studies (e.g., Vanchinathan et al.\,\cite{efficientpm})
ambiguously define the harshness as the closeness to the boundary of the cells.
Furthermore, we show in Lemma \ref{lem_opt} in the Appendix that
$\optcone{1}(p^*,\{p_i^*\})=\lo(
N/\kyorim{p^*-\cP_1^c}^{2})$:
 the regret bound has at most quadratic dependence on $\kyorim{p^*-\cP_1^c}$, which is defined in Appendix \ref{append_main} as the closeness of $\pst$ to the boundary of the optimal cell.

\vspace{-0.7em}
\section{PM-DMED Algorithm}
\label{sec:pmdmed}
\vspace{-1em}

In this section, we describe the partial monitoring deterministic minimum empirical divergence (PM-DMED) algorithm, which is inspired by DMED \cite{HondaDMED} for solving the multi-armed bandit problem. 
Let $\phat_i(t) \in [0,1]^A$ be the empirical distribution of the symbols under the selection of action $i$. Namely, the $k$-th element of $\phat_i(t)$ is $(\sum_{t'=1}^{t-1} \Ind[i(t') = i \cap h_{i(t'), j(t')} = k]) / (\sum_{t'=1}^{t-1} \Ind[i(t')=i])$. We sometimes omit $t$ from $\phat_i$ when it is clear from the context.
Let the empirical divergence of $q \in \simplex$ be
$\sum_{i \in \actions} N_i(t) D(\phat_i(t) \Vert S_i q)$, the exponential of which can be considered as a likelihood that $q$ is the opponent's strategy.

The main routine of PM-DMED is in Algorithm \ref{alg_dmedloop}. 
At each loop, the actions in the current list $Z_C$ are selected once. The list for the actions in the next loop $Z_N$ is determined by the subroutine in Algorithm \ref{alg_pmdmed}. 
The subroutine checks whether the empirical divergence of each point $q \in \nonoptimal$ is larger than $\log{t}$ or not (Eq.\,\eqref{base_cond_tansaku}). 
If it is large enough, it exploits the current information by selecting $\ihat(t)$, the optimal action based on the estimation $\phat(t)$ that minimizes the empirical divergence.
Otherwise, it selects the actions with the number of observations below the minimum requirement for making the empirical divergence of each suboptimal point $q \in \nonoptimal$ larger than $\log{t}$.

Unlike the $N$-armed bandit problem in which a reward is associated with an action, in the partial monitoring problem, actions, outcomes, and feedback signals can be intricately related.
Therefore, we need to solve a non-trivial optimization to run PM-DMED.
Later in Section \ref{sec:experiment}, we discuss a practical implementation of the optimization.

\begin{minipage}[t]{.49\linewidth}
\null 
\vspace{-2.5em}
\begin{algorithm}[H]
\caption{Main routine of PM-DMED and PM-DMED-Hinge}\label{alg_dmedloop}
\begin{algorithmic}[1]
\STATE \textbf{Initialization:} select each action once.
\STATE $Z_C, Z_R \leftarrow [N], Z_N \leftarrow \emptyset$.
\WHILE{$t \le T$}
  \FOR{$i(t) \in Z_C$ in an arbitrarily fixed order}
    \STATE Select $i(t)$, and receive feedback.
    \STATE $Z_R \leftarrow Z_R \setminus \{i(t)\}$.
    \STATE Add actions to $Z_N$ in accordance with $\begin{cases}
      \text{Algorithm \ref{alg_pmdmed}} & \text{(PM-DMED)} \\
      \text{Algorithm \ref{alg_pmdmedhinge}} & \text{(PM-DMED-Hinge)} \\
      \end{cases}$.
    \STATE $t \leftarrow t+1$.
  \ENDFOR
  \STATE $Z_C, Z_R \leftarrow Z_N$, $Z_N \leftarrow \emptyset$.
\ENDWHILE 
\end{algorithmic}
\end{algorithm}%

\end{minipage}%
\begin{minipage}[t]{.02\linewidth}
\vspace{1em}\hspace{3em}
\end{minipage}%
\begin{minipage}[t]{.49\linewidth}
\null
\vspace{-2.5em}
\begin{algorithm}[H]
\caption{PM-DMED subroutine for adding actions to $Z_N$ (without duplication).}\label{alg_pmdmed}
\begin{algorithmic}[1]
\STATE \textbf{Parameter:} $c>0$.
\STATE Compute an arbitrary $\phat(t)$ such that
\vspace{-0.3em}
\begin{equation}
  \phat(t) \hspace{-0.1em} \in \hspace{-0.1em} \argmin_{q}\sum_iN_i(t) D(\phat_i(t)\Vert S_i q) \label{base_def_phat}
\vspace{-0.3em}
\end{equation}
and let $\ihat(t) = \argmin_{i} L_i^\top \phat(t)$.
\STATE If $\ihat(t) \notin Z_R$ then put $\ihat(t)$ into $Z_N$.
\STATE If there are actions $i \notin Z_R$ such that
\vspace{-0.3em}
\begin{align}
&N_i(t)< c \sqrt{\log t}\label{ineq:halflog}
\end{align}
then put them into $Z_N$.
\STATE If
\vspace{-0.3em}
\begin{equation}
  \{N_i(t)/\log t\}_{i \neq \ihat(t)} \notin \admq{\ihat(t)}{\{\phat_i(t)\}}
  \label{base_cond_tansaku}
\vspace{-0.3em}
\end{equation}
then compute some
\vspace{-0.3em}
\begin{align}
\label{base_r_seisitu}
\{r_i^*\}_{i\neq \ihat(t)}
\in \optset{\ihat(t)}{\phat(t),\{\phat_i(t)\}}
\vspace{-0.3em}
\end{align}
and put all actions $i$ such that $i \notin Z_R$ and $r_i^*>N_i(t)/\log t$ into $Z_N$.
\end{algorithmic}
\end{algorithm}%
\end{minipage}

\textbf{Necessity of $\sqrt{\log{T}}$ exploration:}
PM-DMED tries to observe each action to some extent (Eq.\,\eqref{ineq:halflog}), which is necessary for the following reason:
consider a four-state game characterized by
\begin{equation*}
  L = \left(
    \begin{array}{cccc}
      0 & 1 & 1 & 0 \\
      10 & 1 & 0 & 0 \\
      10 & 0 & 1 & 0 \\
      11 & 11 & 11 & 11
    \end{array}
  \right) \text{, \hspace{1em}}
  H = \left(
    \begin{array}{cccc}
      1 & 1 & 1 & 1 \\
      1 & 2 & 2 & 3 \\
      1 & 2 & 2 & 3 \\
      1 & 1 & 2 & 2 
    \end{array}
  \right) \text{, and \hspace{1em}} \pst = (0.1, 0.2, 0.3, 0.4)^{\top}.
\end{equation*}
The optimal action here is action $1$, which does not yield any useful information.
By using action $2$, one receives three kinds of symbols from which one can estimate $(\pst)_1$, $(\pst)_2 + (\pst)_3$, and $(\pst)_4$, where $(\pst)_j$ is the $j$-th component of $\pst$.
From this, an algorithm can find that $(\pst)_1$ is not very small and thus the expected loss of actions $2$ and $3$ is larger than that of action $1$.
Since the feedback of actions $2$ and $3$ are the same, one may also use action $3$ in the same manner. However, the loss per observation is $1.2$ and $1.3$ for actions $2$ and $3$, respectively, and thus it is better to use action $2$. This difference comes from the fact that $(\pst)_2 = 0.2 < 0.3 = (\pst)_3$. Since an algorithm does not know $\pst$ beforehand, it needs to observe action $4$, the only source for distinguishing $(\pst)_2$ from $(\pst)_3$.
Yet, an optimal algorithm cannot select it more than $\Omega(\log{T})$ times because it affects the $\lo(\log{T})$ factor in the regret. 
In fact, $\lo((\log{T})^a)$ observations of action $4$ with some $a>0$ are sufficient to be convinced that $(\pst)_2 < (\pst)_3$ with probability $1-\so(1/T^{\mathrm{poly}(a)})$. 
For this reason, PM-DMED selects each action $\sqrt{\log{t}}$ times.

\section{Experiment}
\label{sec:experiment}

\begin{figure*}[t!]
\vspace{-1em}
\begin{center}
  \setlength{\subfigwidth}{.32\linewidth}
  \addtolength{\subfigwidth}{-.32\subfigcolsep}
  \begin{minipage}[t]{\subfigwidth}
  \centering
 \subfigure[three-states, benign]{
 \includegraphics[scale=0.3]{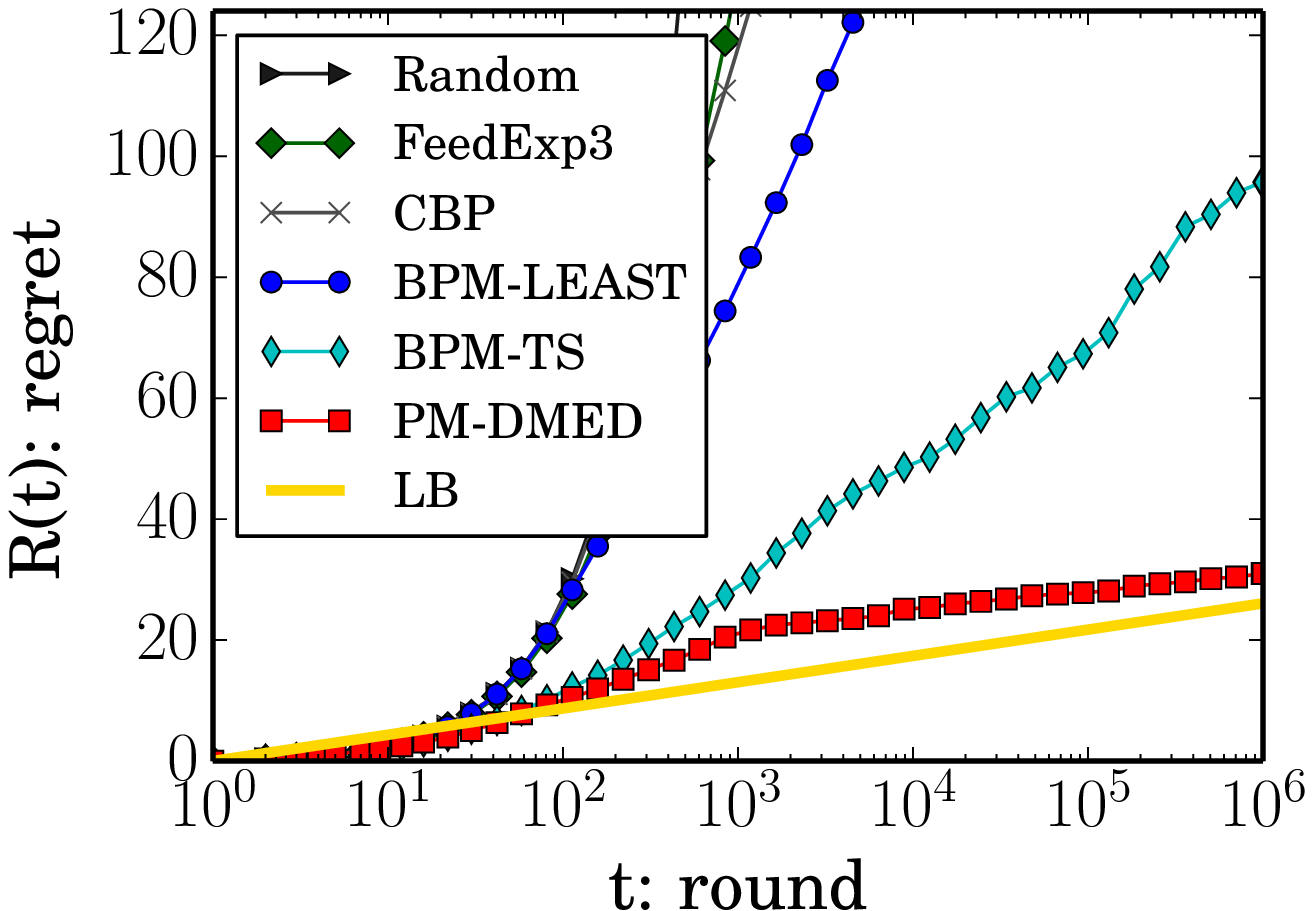}
 }
 \vspace{-0.501em}
 \end{minipage}\hfill
  \begin{minipage}[t]{\subfigwidth}
  \centering
 \subfigure[three-states, intermediate]{\includegraphics[scale=0.3]{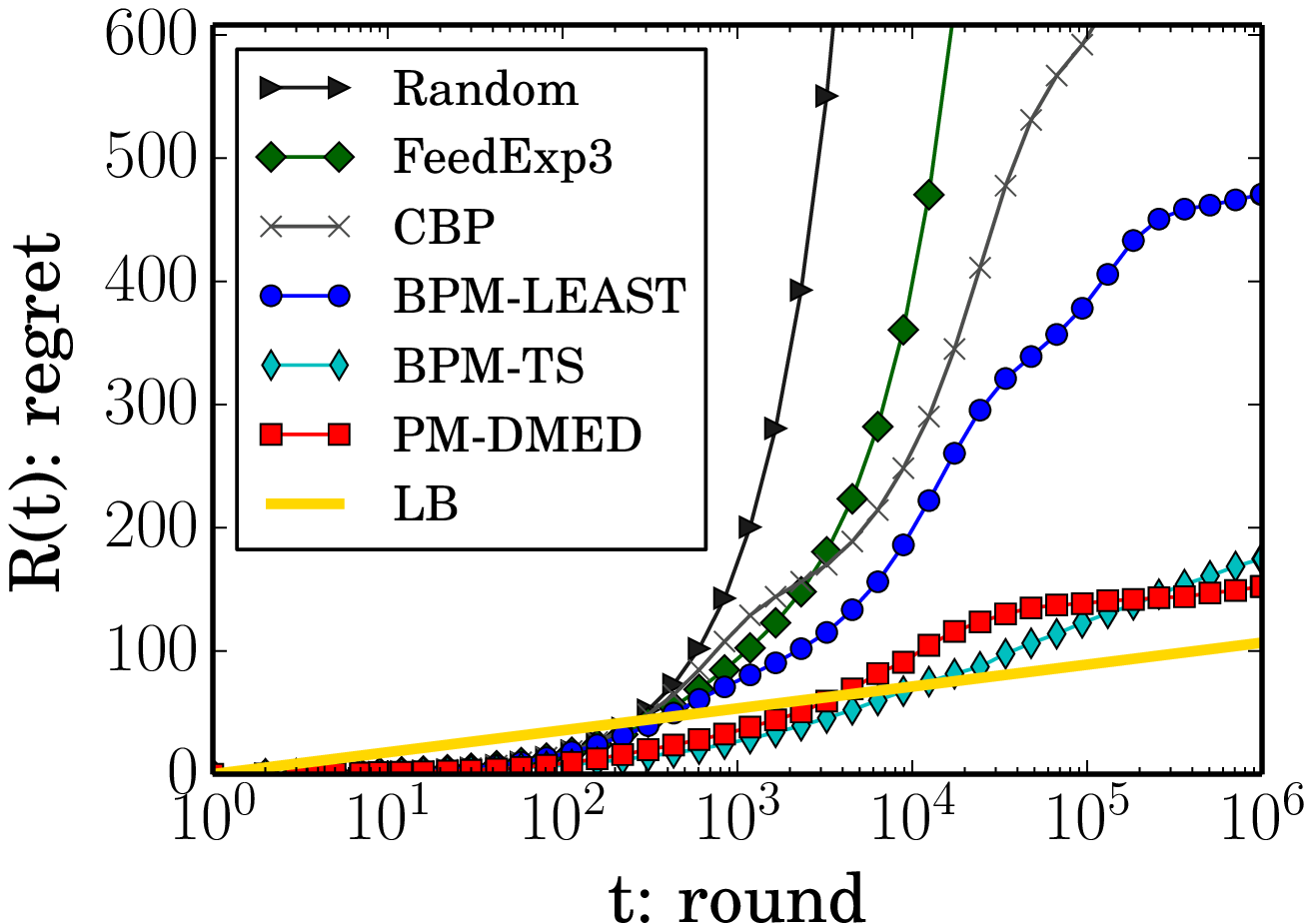}
 }
 \vspace{-0.501em}
  \end{minipage}\hfill
  \begin{minipage}[t]{\subfigwidth}
  \centering
 \subfigure[three-states, harsh]{\includegraphics[scale=0.3]{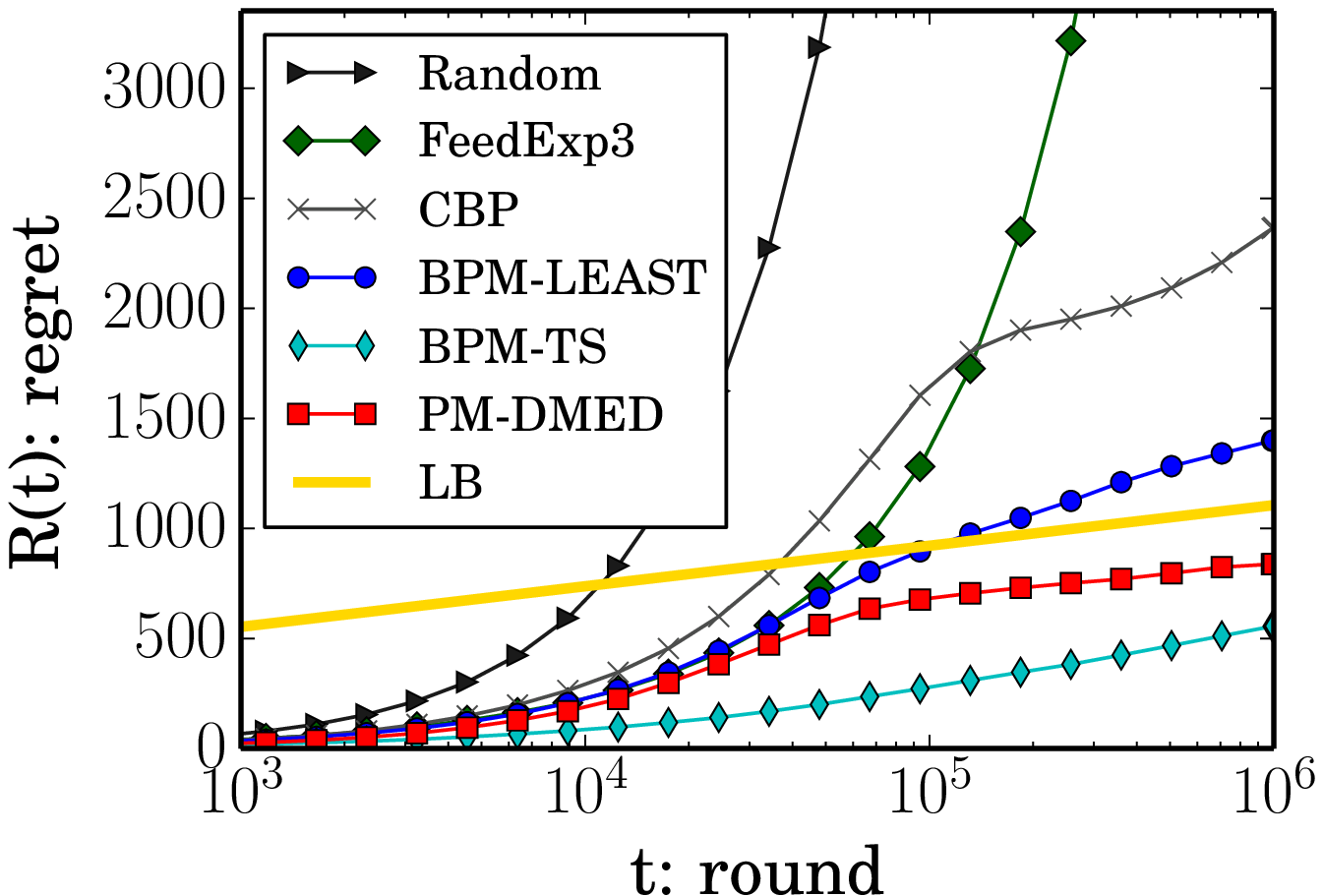}
 }
 \vspace{-0.501em}
  \end{minipage}
  \begin{minipage}[t]{\subfigwidth}
  \centering
 \subfigure[dynamic pricing, benign]{\includegraphics[scale=0.3]{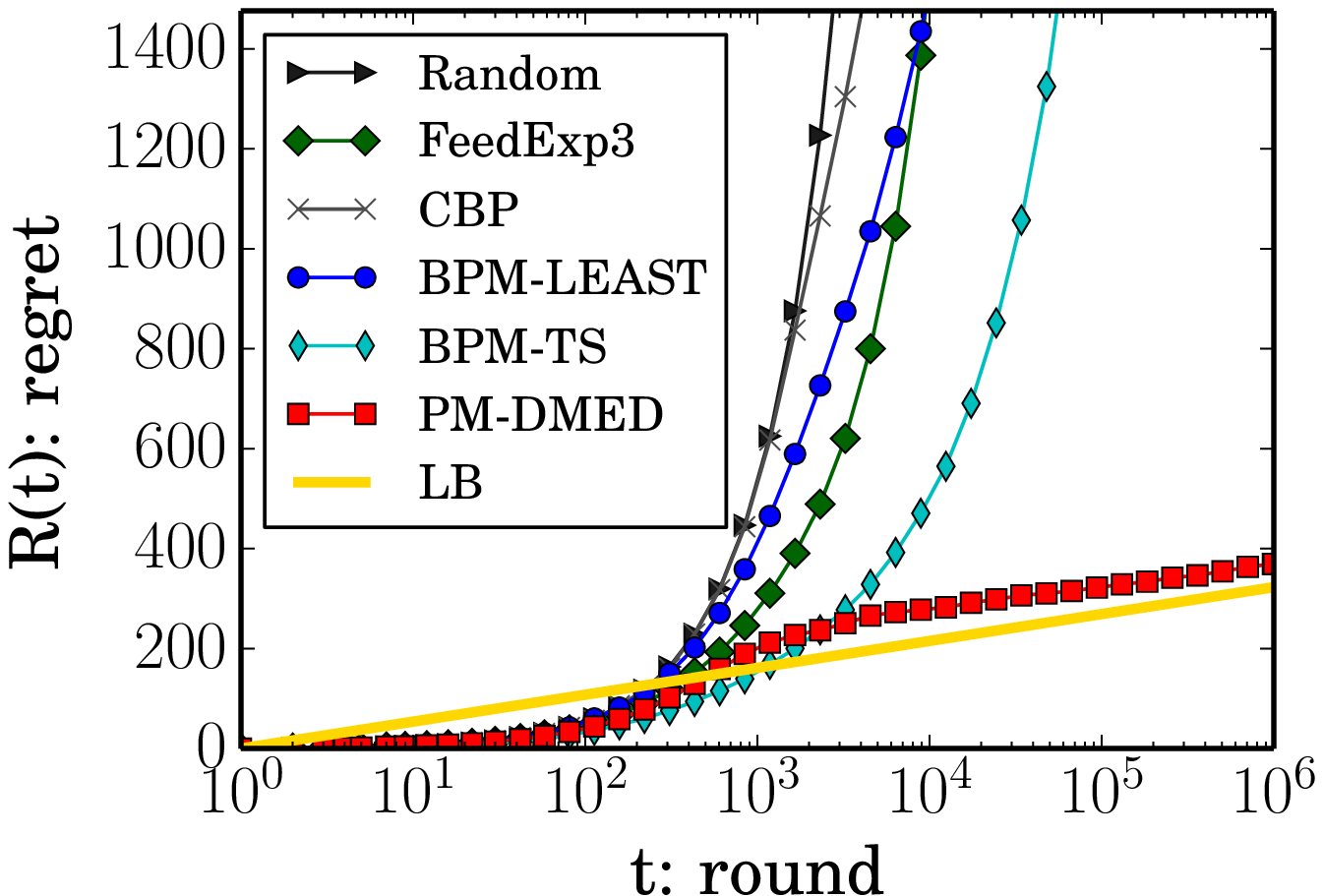}
 }
 \vspace{-0.501em}
 \end{minipage}\hfill
  \begin{minipage}[t]{\subfigwidth}
  \centering
 \subfigure[dynamic pricing, intermediate]{\includegraphics[scale=0.3]{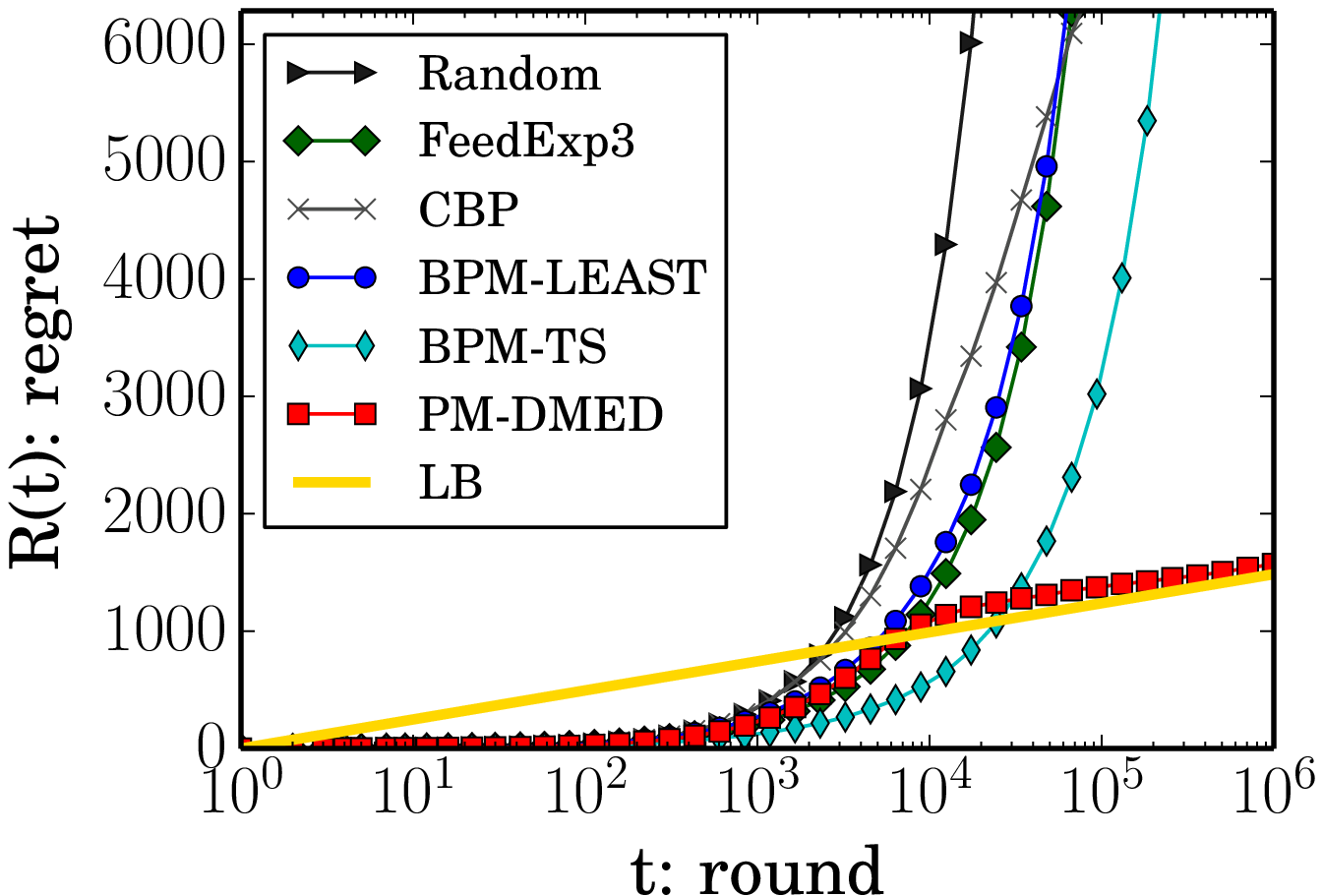}
  }
 \vspace{-0.501em}
  \end{minipage}\hfill
  \begin{minipage}[t]{\subfigwidth}
  \centering
 \subfigure[dynamic pricing, harsh]{\includegraphics[scale=0.3]{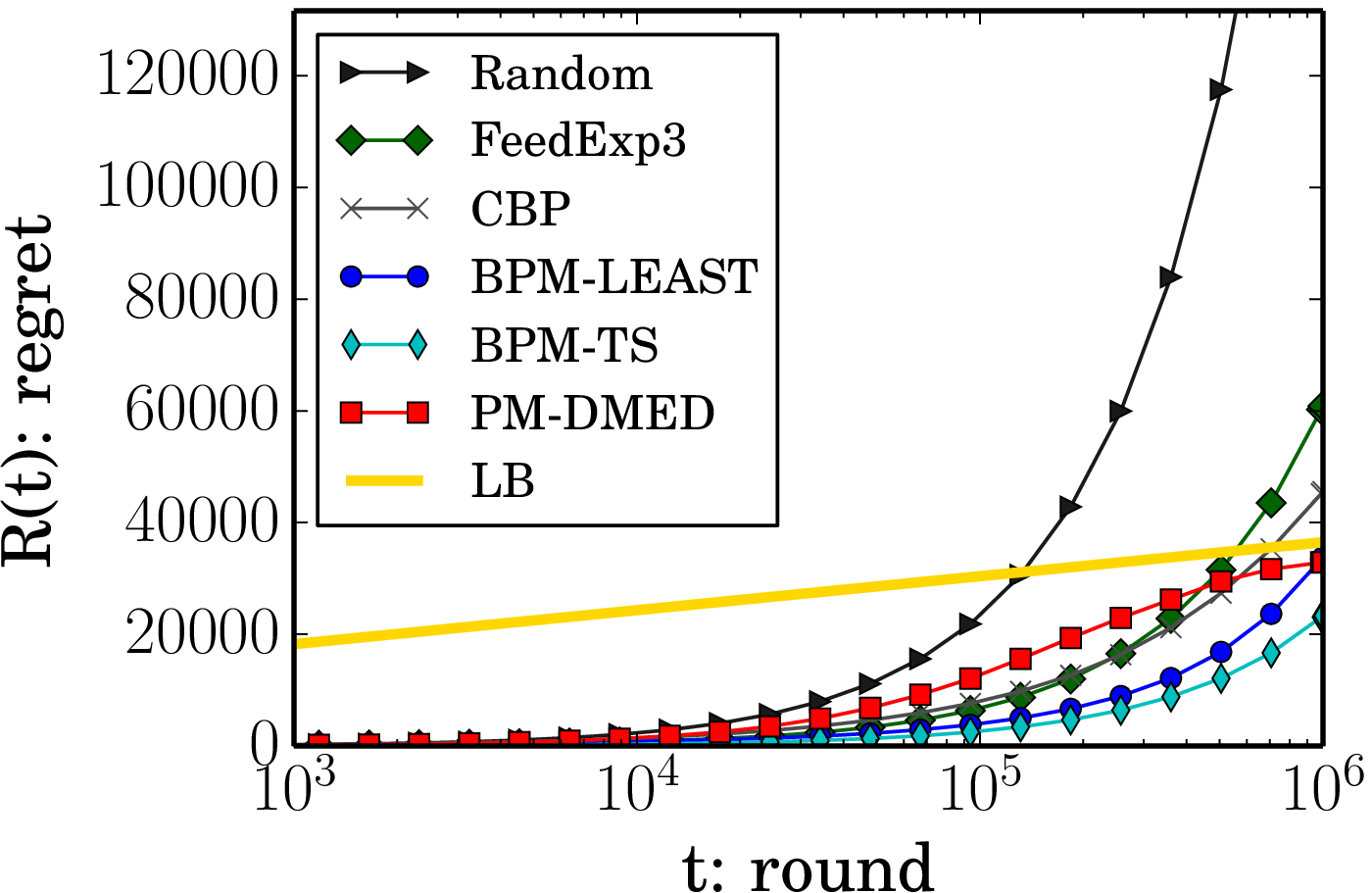}
 }
 \vspace{-0.501em}
  \end{minipage}\hfill
  \begin{minipage}[t]{\subfigwidth}
  \centering
 \subfigure[four-states]{\includegraphics[scale=0.3]{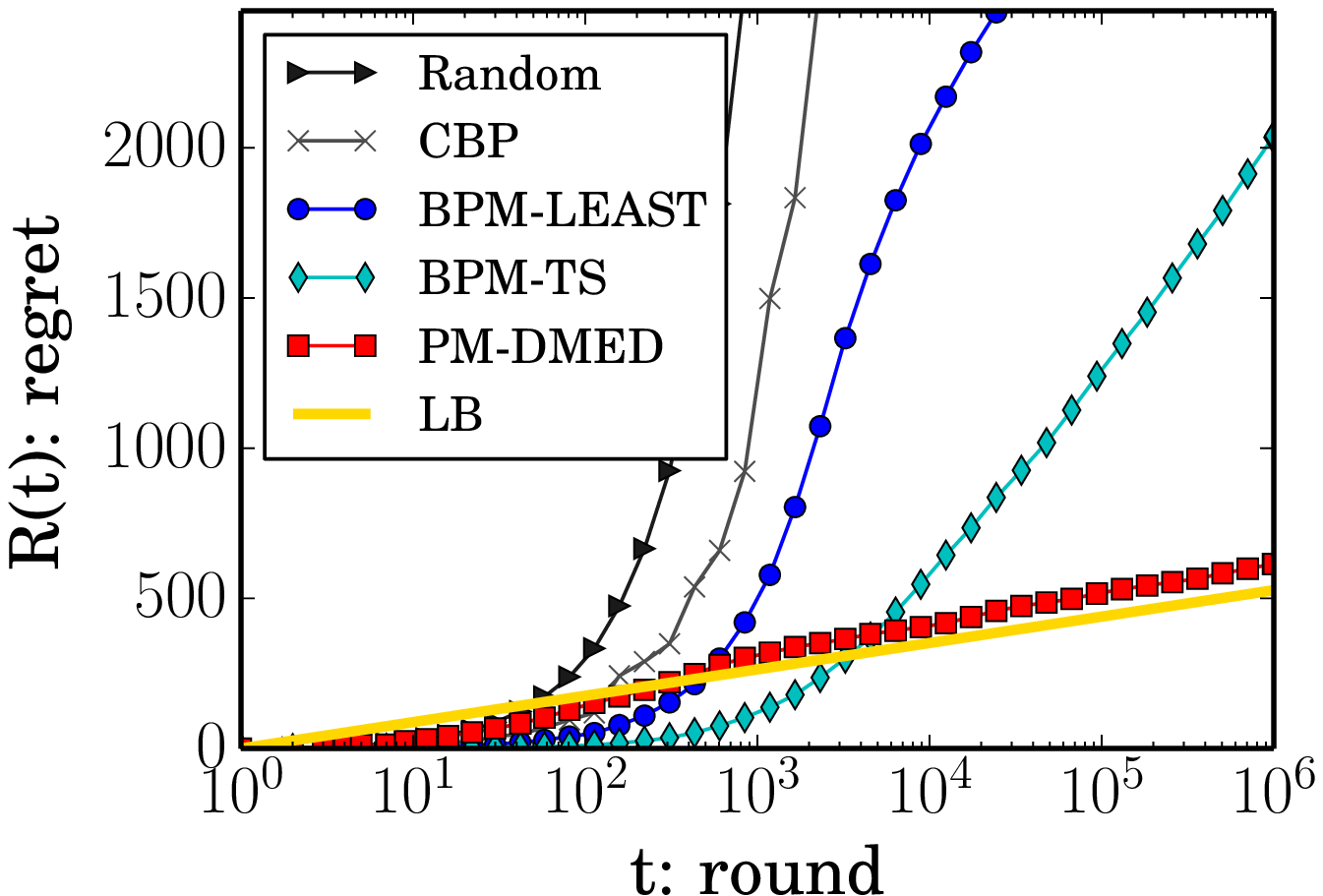}
 }
 \vspace{-0.501em}
  \end{minipage}\hfill
\end{center}
\vspace{-1em}
  \caption{Regret-round semilog plots of algorithms. The regrets are averaged over 100 runs. LB is the asymptotic regret lower bound of Theorem \ref{thm:regretlower}.}
\vspace{-1.0em}
 \label{fig:regret1}
\end{figure*}%

Following Bart{\'{o}}k et al.\,\cite{bartokicml12}, we compared the performances of algorithms in three different games: the four-state game (Section \ref{sec:pmdmed}), a three-state game and dynamic pricing. Experiments on the $N$-armed bandit game was also done, and the result is shown in Appendix \ref{subsec:banditexperiment}.

The three-state game, which is classified as easy in terms of the minimax regret, is characterized by:
\vspace{-0.1em}
\begin{equation*}
  L = \left(
    \begin{array}{ccc}
      1 & 1 & 0 \\
      0 & 1 & 1 \\
      1 & 0 & 1
    \end{array}
  \right) \text{\hspace{1em} and \hspace{1em}}
  H = \left(
    \begin{array}{ccc}
      1 & 2 & 2 \\
      2 & 1 & 2 \\
      2 & 2 & 1
    \end{array}
  \right).
\vspace{-0.1em}
\end{equation*}
The signal matrices of this game are,
\vspace{-0.1em}
\begin{equation*}
  \signali{1} = \left(
    \begin{array}{ccc}
      1 & 0 & 0 \\
      0 & 1 & 1
    \end{array}
  \right),  \text{\hspace{1em}}
  \signali{2} = \left(
    \begin{array}{ccc}
      0 & 1 & 0 \\
      1 & 0 & 1
    \end{array}
  \right),  \text{ and }
  \signali{3} = \left(
    \begin{array}{ccc}
      0 & 0 & 1 \\
      1 & 1 & 0
    \end{array}
  \right).
\vspace{-0.1em}
\end{equation*}

Dynamic pricing, which is classified as hard in terms of the minimax regret, is a game that models a repeated auction between a seller (learner) and a buyer (opponent).
At each round, the seller sets a price for a product, and at the same time, the buyer secretly sets a maximum price he is willing to pay. The signal is ``buy'' or ``no-buy'', and the seller's loss is either a given constant (no-buy) or the difference between the buyer's and the seller's prices (buy). The loss and feedback matrices are:
\begin{equation*}
  L = \left(
    \begin{array}{cccc}
      0      &    1   & \ldots & N-1 \\
      c      &    0   & \ldots & N-2 \\
      \vdots & \ddots & \ddots & \vdots \\
      c      & \ldots &    c   &     0
    \end{array}
  \right) \text{\hspace{1em} and \hspace{1em}}
  H = \left(
    \begin{array}{cccc}
      2      &    2   & \ldots &    2   \\
      1      &    2   & \ldots &    2   \\
      \vdots & \ddots & \ddots & \vdots \\
      1      & \ldots &    1   &    2
    \end{array}
  \right), 
\end{equation*}
where signals $1$ and $2$ correspond to no-buy and buy.
The signal matrix of action $i$ is 
\begin{equation*}
\vspace{-0.5em}
   \signali{i} = \Bigl( \overbrace{ 
    \begin{array}{cc}
      1 \hspace{1em}\ldots\hspace{1em} 1  \\
      0 \hspace{1em}\ldots\hspace{1em}  0  
    \end{array}
  }^{i-1} 
  \overbrace{ 
    \begin{array}{cc}
      0 \hspace{1em}\ldots\hspace{1em} 0  \\
      1 \hspace{1em}\ldots\hspace{1em} 1   
    \end{array}
  }^{M-i+1} \Bigr)
.
\end{equation*}
Following Bart{\'{o}}k et al.\,\cite{bartokicml12}, we set $N=5, M=5$, and $c=2$.

In our experiments with the three-state game and dynamic pricing, we tested three settings regarding the harshness of the opponent:
 at the beginning of a simulation,
we sampled 1,000 points uniformly at random from $\simplex$,
  then sorted them by $\optcone{1}(p^*,\{p_i^*\})$.
We chose the top 10\%, 50\%, and 90\% harshest ones as the opponent's strategy in the harsh, intermediate, and benign settings, respectively. 

We compared Random, FeedExp3 \cite{PiccolboniChristian2001}, CBP \cite{bartokicml12} with $\alpha=1.01$, BPM-LEAST, BPM-TS \cite{efficientpm}, and PM-DMED with $c=1$.
Random is a naive algorithm that selects an action uniformly random.
FeedExp3 requires a matrix $G$ such that $H^\top G = L^\top$, and thus one cannot apply it to the four-state game. CBP is an algorithm of logarithmic regret for easy games. The parameters $\eta$ and $f(t)$ of CBP were set in accordance with Theorem 1 in their paper. BPM-LEAST is a Bayesian algorithm with $\tilde{O}(\sqrt{T})$ regret for easy games, and BPM-TS is a heuristic of state-of-the-art performance. The priors of two BPMs were set to be uninformative to avoid a misspecification, as recommended in their paper.

The computation of $\phat(t)$ in \eqref{base_def_phat} and the evaluation of the condition in \eqref{base_cond_tansaku} involve convex optimizations, which were done with Ipopt \cite{ipopt}. 
Moreover, obtaining $\{r_i^*\}$ in \eqref{base_r_seisitu} is classified as a linear semi-infinite programming (LSIP) problem, a linear programming (LP) with finitely many variables and infinitely many constraints. Following the optimization of BPM-LEAST \cite{efficientpm}, we resorted to a finite sample approximation and used the Gurobi LP solver \cite{gurobi} in computing $\{r_i^*\}$: at each round, we sampled 1,000 points from $\simplex$, and relaxed the constraints on the samples.
To speed up the computation, we skipped these optimizations in most rounds with large $t$ and used the result of the last computation.
The computation of the coefficient $\optcone{1}(p^*,\{p_i^*\})$ of the regret lower bound (Theorem \ref{thm:regretlower}) is also an LSIP, which was approximated by 100,000 sample points from $\nonoptimal$. 

The experimental results are shown in Figure \ref{fig:regret1}. In the four-state game and the other two games with an easy or intermediate opponent, PM-DMED outperforms the other algorithms when the number of rounds is large. 
 In particular, in the dynamic pricing game with an intermediate opponent, the regret of PM-DMED at $T=10^6$ is ten times smaller than those of the other algorithms.
Even in the harsh setting in which the minimax regret matters, PM-DMED has some advantage over all algorithms except for BPM-TS. 
With sufficiently large $T$, the slope of an optimal algorithm should converge to LB. In all games and settings, the slope of PM-DMED converges to LB, which is empirical evidence of the optimality of PM-DMED.

\vspace{-0.8em}
\section{Theoretical Analysis}
\vspace{-1em}
\label{sec:analysis}

\begin{algorithm}[t]
\caption{PM-DMED-Hinge subroutine for adding actions to $Z_N$ (without duplication).}\label{alg_pmdmedhinge}
\begin{algorithmic}[1]
\STATE \textbf{Parameters:} $c>0$, $f(n) = b n^{-1/2}$ for $b>0$, $\alpha(t) = a / (\log\log t)$ for $a>0$.
\STATE Compute arbitrary $\phat(t)$ which satisfies
\vspace{-0.3em}
\begin{equation}
\phat(t) \in \argmin_{q} \sum_i N_i(t) (D(\phat_i(t)\Vert S_i q)-f(N_i(t)))_+\label{def_phat}
\vspace{-0.3em}
\end{equation}
and let $\ihat(t) = \argmin_{i} L_i^\top \phat(t)$.
\STATE If $\ihat(t) \notin Z_R$ then put $\ihat(t)$ into $Z_N$.
\STATE If 
\vspace{-0.3em}
\begin{equation}
 \phat(t) \notin \cP_{\ihat(t),\al(t)}
 \label{closetoboundary}
\vspace{-0.3em}
\end{equation}
or there exists an action $i$ such that
\vspace{-0.3em}
\begin{equation}
D(\phat_i(t)\Vert S_i \phat(t))> f(N_i(t))\label{zyouken_syuusoku}
\vspace{-0.3em}
\end{equation}
then put all actions $i\notin Z_R$ into $Z_N$.
\STATE If there are actions $i$ such that
\vspace{-0.3em}
\begin{equation}
 N_i(t)< c \sqrt{\log t}\label{cond_sqrtlog}
\vspace{-0.3em}
\end{equation}
then put the actions not in $Z_R$ into $Z_N$.
\STATE If
\vspace{-0.3em}
\begin{equation}
\{N_i(t)/\log t\}_{i\neq \ihat(t)}\notin \adm{\ihat(t)}{\{\phat_i(t),\,f(N_i(t))\}}
\vspace{-0.3em}
\label{cond_tansaku}
\end{equation}
then compute some
\vspace{-0.3em}
\begin{align}
\{r_i^*\}_{i\neq \ihat(t)}\in \optset{\ihat(t)}{\phat(t), \{\phat_i(t),\,f(N_i(t))\}}\label{r_seisitu}
\vspace{-0.3em}
\end{align}
and put all actions such that $i \notin Z_R$ and $r_i^*>N_i(t)/\log t$ into $Z_N$.
If such $r_i^*$ is infeasible then put all action $i\notin Z_R$ into $Z_N$.
\end{algorithmic}
\end{algorithm}%
Section \ref{sec:experiment} shows that the empirical performance of PM-DMED is very close to the regret lower bound in Theorem \ref{thm:regretlower}. Although the authors conjecture that PM-DMED is optimal, it is hard to analyze PM-DMED.
The technically hardest part arises from the case in which the divergence of each action is small but not yet fully converged. To circumvent this difficulty, we can introduce a discount factor.
Let 
\vspace{-0.2em}
\begin{equation}
\adm{j}{\{p_i,\de_i\}}
\hspace{-0.2em} = \hspace{-0.3em}
\biggl\{
\{r_i\}_{i\neq j}\in [0,\infty)^{N-1}:
\hspace{-0.5em}\inf_{q\in \closure(\cP_{j}^c):
D(p_j\Vert S_j q)\le \de_j}
\sum_{i} r_i(D(p_i\Vert S_i q)-\de_i)_+\ge 1
\biggr\},
\label{feasible}
\vspace{-0.2em}
\end{equation}
where $(X)_+ = \max(X, 0)$.
Note that $\adm{j}{\{p_i,\de_i\}}$ in \eqref{feasible} is
a natural generalization of $\admq{j}{\{p_i\}}$ in Section \ref{sec:pmdmed}
in the sense that $\adm{j}{\{p_i,0\}} = \admq{j}{\{p_i\}}$.
Event $\{N_i(t)/\log t\}_{i \neq 1} \in \adm{1}{\{\phat_i(t),\de_i\}}$ means that
the number of observations $\{N_i(t)\}$ is enough to ensure that the ``$\{\de_i\}$-discounted'' empirical divergence of each $q \in \nonoptimal$ is larger than $\log{t}$.
Analogous to $\adm{j}{\{p_i,\de_i\}}$, we define
\vspace{-0.2em}
\begin{align*}
\optc{j}{p,\{p_i,\de_i\}}&=
\inf_{\{r_i\}_{i\neq j}\in \adm{j}{\{p_i,\de_i\})}}
\sum_{i\neq j} r_i (L_j-L_i)^\top p
\vspace{-0.2em}
\end{align*}
and its optimal solution by
\vspace{-0.2em}
\begin{align*}
\optset{j}{p,\{p_i,\de_i\}}&=\biggl\{
\{r_i\}_{i\neq j}\in \adm{j}{\{p_i,\de_i\}}:
\sum_{i\neq j} r_i (L_j-L_i)^\top p
=
\optc{j}{p,\{p_i,\de_i\}}
\biggr\}\per\n
\vspace{-0.2em}
\end{align*}
We also define
$\cP_{i,\al}=\{p\in \simplex: L_i^\top p + \al \le \min_{j\neq i}L_j^\top p \}
$, the optimal region of action $i$ with margin.
PM-DMED-Hinge shares the main routine of Algorithm \ref{alg_dmedloop} with PM-DMED and lists the next actions by Algorithm \ref{alg_pmdmedhinge}.  
Unlike PM-DMED, it (i) discounts $f(N_i(t))$ from the empirical divergence $D(\phat_i(t)\Vert S_i q)$. 
Moreover, (ii) when $\phat(t)$ is close to the cell boundary, it encourages more exploration to identify the cell it belongs to by Eq.\,\eqref{closetoboundary}.

\begin{theorem}
Assume the following regularity conditions hold for $p^*$.
(1) $\optset{1}{p,\{p_i,\de_i\}}$ is unique at $p=p^*,p_i=S_i p^*,\,\de_i=0$.
Moreover, (2) for $\setd_{\de}=\{q :D(p_1^* \Vert S_1 q)\le \de\}$,
it holds that
$\closure(\interior(\cP_1^c)\cap \setd_{\de})=\closure(\closure(\cP_1^c)\cap \setd_{\de})$
for all $\de\ge 0$ in some neighborhood of $\de=0$,
where
$\closure(\cdot)$ and $\interior(\cdot)$ denote the closure and the interior,
respectively.
Then, 
\begin{align*}
\E[\regret(T)]\le
\optc{1}{p^*,\{p^*_i\}} \log T\n
+\so(\log T)\per
\end{align*}
\label{thm_main}
\end{theorem}%
We prove this theorem in Appendix \ref{append_main}.
Recall that $\optset{1}{p,\{\phat_i(t),\de_i\}}$ is the set of
optimal solutions of an LSIP.
In this problem, KKT conditions and the duality theorem apply as in the case of finite constraints; thus, we can check whether Condition 1 holds or not
for each $p^*$ (see, e.g., Ito et al.\,\cite{ito_semi_infinite} and references therein).
Condition 2 holds in most cases, and
an example of an exceptional case is shown in Appendix \ref{sec:cornercase}.

Theorem \ref{thm_main} states that PM-DMED-Hinge has a regret upper bound that matches the lower bound of Theorem \ref{thm:regretlower}.

\begin{corollary} (Optimality in the $N$-armed bandit problem)
In the $N$-armed Bernoulli bandit problem, the regularity conditions in Theorem \ref{thm_main} always hold. Moreover, the coefficient of the leading logarithmic term in the regret bound of the partial monitoring problem is equal to the bound given in Lai and Robbins \cite{LaiRobbins1985}. Namely,
$\optcone{1}(p^*,\{p_i^*\}) = \sum_{i \neq 1}^N (\Delta_i/d(\mu_i \Vert \mu_1))$,
\label{cor:bandit}%
where $d(p \Vert q) = p \log{(p/q)} + (1-p) \log{((1-p)/(1-q))}$ is the KL-divergence between Bernoulli distributions.
\vspace{-0.2em}
\end{corollary}%
Corollary \ref{cor:bandit}, which is proven in Appendix \ref{sec:banditproof}, states that 
PM-DMED-Hinge attains the optimal regret of the $N$-armed bandit if we run it on an $N$-armed bandit game represented as partial monitoring.

\textbf{Asymptotic analysis:} it is Theorem \ref{thm_conti} where we 
lose the finite-time property.
This theorem shows the continuity of the optimal solution set
$\optset{1}{p,\{p_i,\de_i\}}$ of $\optcone{j}(p,\{p_j\})$,
  which does not mention how close 
$\optset{1}{p,\{p_i,\de_i\}}$ is to
$\optset{1}{p^*,\{p_i^*,0\}}$
if $\max\{\kyorim{p-p^*},\max_i\kyorim{p_i-p_i^*},\max_i\de_i\}\le \de$
for given $\de$.
To obtain an explicit bound,
we need \textit{sensitivity analysis}, the theory of the robustness of the optimal value and the solution for small deviations of its parameters
(see e.g., Fiacco \cite{fiacco}).
In particular, the optimal solution of partial monitoring involves an infinite number of constraints, which makes the analysis quite hard.
For this reason, we will not perform a finite-time analysis.
Note that, the $N$-armed bandit problem is a special instance in which we can avoid solving the above optimization and a finite-time optimal bound is known.

\textbf{Necessity of the discount factor:} we are not sure whether discount factor $f(n)$ in PM-DMED-Hinge is necessary or not. We also empirically tested PM-DMED-Hinge: although it is better than the other algorithms in many settings, such as dynamic pricing with an intermediate opponent, it is far worse than PM-DMED. 
We found that our implementation, which uses the Ipopt nonlinear optimization solver, was sometimes inaccurate at optimizing \eqref{def_phat}: there were some cases in which the true $\pst$ satisfies $\forall_{i \in \actions} D(\phat_i(t)\Vert S_i p^*)-f(N_i(t)) = 0$, while the solution $\phat(t)$ we obtained had non-zero hinge values.
In this case, the algorithm lists all actions from \eqref{zyouken_syuusoku}, which degrades performance.
Determining whether the discount factor is essential or not is our future work.

\vspace{-1.0em}
\section*{Acknowledgements}
\vspace{-1.0em}

The authors gratefully acknowledge the advice of Kentaro Minami and sincerely thank the anonymous reviewers for their useful comments. This work was supported in part by JSPS KAKENHI Grant Number 15J09850 and 26106506.

\clearpage
\bibliographystyle{unsrt}
\bibliography{bibs/manual}

\begin{thebibliography}{10}

\bibitem{Littlestone:1994:WMA:184036.184040}
Nick Littlestone and Manfred~K. Warmuth.
\newblock The weighted majority algorithm.
\newblock {\em Inf. Comput.}, 108(2):212--261, February 1994.

\bibitem{LaiRobbins1985}
T.~L. Lai and Herbert Robbins.
\newblock Asymptotically efficient adaptive allocation rules.
\newblock {\em Advances in Applied Mathematics}, 6(1):4--22, 1985.

\bibitem{DBLP:conf/focs/KleinbergL03}
Robert~D. Kleinberg and Frank~Thomson Leighton.
\newblock The value of knowing a demand curve: Bounds on regret for online
  posted-price auctions.
\newblock In {\em FOCS}, pages 594--605, 2003.

\bibitem{DBLP:journals/jmlr/AgarwalBD10}
Alekh Agarwal, Peter~L. Bartlett, and Max Dama.
\newblock Optimal allocation strategies for the dark pool problem.
\newblock In {\em AISTATS}, pages 9--16, 2010.

\bibitem{DBLP:journals/tit/Cesa-BianchiLS05}
Nicol{\`{o}} Cesa{-}Bianchi, G{\'{a}}bor Lugosi, and Gilles Stoltz.
\newblock Minimizing regret with label efficient prediction.
\newblock {\em {IEEE} Transactions on Information Theory}, 51(6):2152--2162,
  2005.

\bibitem{DBLP:conf/icml/Zinkevich03}
Martin Zinkevich.
\newblock Online convex programming and generalized infinitesimal gradient
  ascent.
\newblock In {\em ICML}, pages 928--936, 2003.

\bibitem{DBLP:conf/colt/DaniHK08}
Varsha Dani, Thomas~P. Hayes, and Sham~M. Kakade.
\newblock Stochastic linear optimization under bandit feedback.
\newblock In {\em COLT}, pages 355--366, 2008.

\bibitem{PiccolboniChristian2001}
Antonio Piccolboni and Christian Schindelhauer.
\newblock Discrete prediction games with arbitrary feedback and loss.
\newblock In {\em COLT}, pages 208--223, 2001.

\bibitem{DBLP:journals/mor/Cesa-BianchiLS06}
Nicol{\`{o}} Cesa{-}Bianchi, G{\'{a}}bor Lugosi, and Gilles Stoltz.
\newblock Regret minimization under partial monitoring.
\newblock {\em Math. Oper. Res.}, 31(3):562--580, 2006.

\bibitem{DBLP:journals/jmlr/BartokPS11}
G{\'{a}}bor Bart{\'{o}}k, D{\'{a}}vid P{\'{a}}l, and Csaba Szepesv{\'{a}}ri.
\newblock Minimax regret of finite partial-monitoring games in stochastic
  environments.
\newblock In {\em COLT}, pages 133--154, 2011.

\bibitem{bartokicml12}
G{\'{a}}bor Bart{\'{o}}k, Navid Zolghadr, and Csaba Szepesv{\'{a}}ri.
\newblock An adaptive algorithm for finite stochastic partial monitoring.
\newblock In {\em ICML}, 2012.

\bibitem{bartokcolt13}
G{\'{a}}bor Bart{\'{o}}k.
\newblock A near-optimal algorithm for finite partial-monitoring games against
  adversarial opponents.
\newblock In {\em COLT}, pages 696--710, 2013.

\bibitem{efficientpm}
Hastagiri~P. Vanchinathan, G{\'{a}}bor Bart{\'{o}}k, and Andreas Krause.
\newblock Efficient partial monitoring with prior information.
\newblock In {\em NIPS}, pages 1691--1699, 2014.

\bibitem{auerfinite}
Peter Auer, Nicol{\'o} Cesa-bianchi, and Paul Fischer.
\newblock {Finite-time Analysis of the Multiarmed Bandit Problem}.
\newblock {\em Machine Learning}, 47:235--256, 2002.

\bibitem{GarivierKLUCB}
Aur{\'{e}}lien Garivier and Olivier Capp{\'{e}}.
\newblock The {KL-UCB} algorithm for bounded stochastic bandits and beyond.
\newblock In {\em COLT}, pages 359--376, 2011.

\bibitem{dembozeitouni}
Amir Dembo and Ofer Zeitouni.
\newblock {\em Large deviations techniques and applications}.
\newblock Applications of mathematics. Springer, New York, Berlin, Heidelberg,
  1998.

\bibitem{HondaDMED}
Junya Honda and Akimichi Takemura.
\newblock {An Asymptotically Optimal Bandit Algorithm for Bounded Support
  Models}.
\newblock In {\em COLT}, pages 67--79, 2010.

\bibitem{ipopt}
Andreas W{\"{a}}chter and Carl~D. Laird.
\newblock Interior point optimizer ({IPOPT}).

\bibitem{gurobi}
Gurobi~Optimization Inc.
\newblock Gurobi optimizer.

\bibitem{ito_semi_infinite}
S.~Ito, Y.~Liu, and K.~L. Teo.
\newblock A dual parametrization method for convex semi-infinite programming.
\newblock {\em Annals of Operations Research}, 98(1-4):189--213, 2000.

\bibitem{fiacco}
Anthony~V. Fiacco.
\newblock {\em Introduction to sensitivity and stability analysis in nonlinear
  programming}.
\newblock Academic Press, New York, 1983.

\bibitem{gravlai97}
T.L. Graves and T.L. Lai.
\newblock Asymptotically efficient adaptive choice of control laws in
  controlled {M}arkov chains.
\newblock {\em SIAM J. Contr. and Opt.}, 35(3):715--743, 1997.

\bibitem{DBLP:conf/colt/KaufmannCG14}
Emilie Kaufmann, Olivier Capp{\'{e}}, and Aur{\'{e}}lien Garivier.
\newblock On the complexity of {A/B} testing.
\newblock In {\em COLT}, pages 461--481, 2014.

\bibitem{bubeckthesis}
S{\'e}bastien Bubeck.
\newblock {\em {Bandits Games and Clustering Foundations}}.
\newblock Theses, {Universit{\'e} des Sciences et Technologie de Lille - Lille
  I}, June 2010.

\bibitem{hogan}
William~W. Hogan.
\newblock Point-to-set maps in mathematical programming.
\newblock {\em SIAM Review}, 15(3):591--603, 1973.

\bibitem{CoverThomasSnd}
Thomas~M. Cover and Joy~A. Thomas.
\newblock {\em Elements of Information Theory, Second Edition (Wiley Series in
  Telecommunications and Signal Processing)}.
\newblock Wiley-Interscience, 2006.

\end{thebibliography}
\clearpage

\appendix

\section{Case in which Condition 2 Does Not Hold}
\label{sec:cornercase}

\begin{figure}[t]
\begin{center}
\centerline{\includegraphics[scale=0.25]{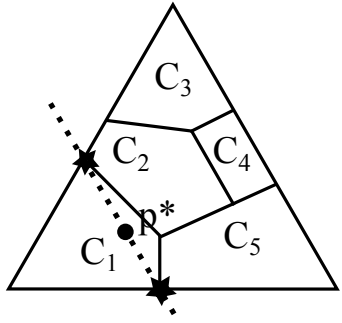}}
\caption{A corner case.}
\label{fig:celldecompexpection}
\end{center}
\end{figure}

Figure \ref{fig:celldecompexpection} is an example that Theorem \ref{thm_main} does not cover.
The dotted line is $\{q: p_1^* = S_1 q\}$,
which (accidentally) coincides with a line that makes the convex polytope of $\cP_1^c$.
In this case, Condition 2 does not hold because $\interior(\cP_1^c)\cap \setd_{0} = \emptyset$ whereas $\closure(\cP_1^c)\cap \setd_{0} \neq \emptyset$ (two starred points),
which means that a slight modification of $\pst$ changes
the set of cells that intersects with the dotted line discontinuously.
We exclude these unusual cases for the ease of analysis.

The authors consider that it is quite hard to give the optimal regret bound without such regularity conditions. In fact, many regularity conditions are assumed in Graves and Lai \cite{gravlai97}, where another generalization of the bandit problem is considered and the regret lower bound is expressed in terms of LSIP. In this paper, the regularity conditions are much simplified by the continuity argument in Theorem 6 but it remains an open problem to fully remove them.

\section{Proof: Regret Lower Bound}
\label{sec:regretlowerproof}

In this section, we prove Lemma \ref{lem:drawlower} and Theorem \ref{thm:regretlower}.

\begin{proof}[Proof of Lemma \ref{lem:drawlower}]
The technique here is mostly inspired from Theorem 1 in Lai and Robbins \cite{LaiRobbins1985}. 
The use of a $\sqrt{T}$ term is inspired from Kaufmann et al.\,\cite{DBLP:conf/colt/KaufmannCG14}.
Let $\pmd \in \interior{(\nonoptimal)}$ and $\imd \neq 1$ be the optimal action under the opponent's strategy $\pmd$.
We consider a modified partial monitoring game with its opponent's strategy is $\pmd$.

\noindent\textbf{Notation:}
Let $\Obsi^m \in [A]$ is the signal of the $m$-th observation of action $i$.
Let
\begin{equation*}
  \hKL_i(n) = \sum_{m=1}^{n} \log{\left(\frac{ (S_i \pst)_{\Obsi^m} }{(S_i \pmd)_{\Obsi^m} }\right)},
\end{equation*}
 and $\hKL = \sum_{i \in \actions} \hKL_i(\nib{i})$. 
 Let $\Probp$ and $\Expectp$ be the probability and the expectation with respect to the modified game, respectively.
Then, for any event $\EE$,
\begin{equation}
  \Probp[\EE] = \Expect\left[\Ind[\EE] \exp{\left(-\hKL\right)}\right] \label{ineq:changemeasure}
\end{equation}
holds.
Now, let us define the following events:
\begin{align*}
 \ED_1 & = \left\{ \sum_{i \in \actions} \nib{i} D(S_i \pst \Vert S_i\pmd) < (1-\epsilon) \log{T}, \nib{\imd} < \sqrt{T}  \right\}, \nn
 \ED_2 & = \left\{ \hKL \leq \left(1 - \frac{\epsilon}{2}\right) \log{T} \right\}, \nn
 \ED_{12} & = \ED_1 \cap \ED_2, \nn
 \ED_{1\backslash2} & = \ED_1 \cap \ED_2^c.
\end{align*}

\noindent\textbf{First step ($\Prob[\ED_{12}] = \so(1)$):} from \eqref{ineq:changemeasure},
\begin{align*}
 \Probp[\ED_{12}]
  & \geq \Expect\left[\Ind[\ED_{12}] \exp{\left(- \left(1 - \frac{\epsilon}{2}\right) \log{T}\right)}\right]  = T^{-(1-\epsilon/2)} \Prob[\ED_{12}].
\end{align*}
By using this we have
\begin{align}
 \Prob[\ED_{12}]
  & \leq T^{(1-\epsilon/2)}\Probp[\ED_{12}] \nn
  & \leq T^{(1-\epsilon/2)}\Probp\left[ \nib{\imd} < \sqrt{T} \right] \nn
  & = T^{(1-\epsilon/2)}\Probp\left[T - \nib{\imd} > T - \sqrt{T} \right] \nn
  & \leq T^{(1-\epsilon/2)}\frac{ \Expectp[T - \nib{\imd}] }{ T -  \sqrt{T}} \text{\hspace{2em} (by the Markov inequality)}. \label{ineq:edbound}
\end{align}
Since this algorithm is strongly consistent, $\Expectp[T - \nib{\imd} ] \rightarrow \so(T^a)$ for any $a>0$.
Therefore, the RHS of the last line of \eqref{ineq:edbound} is $\so(T^{a-\epsilon/2})$, which, by choosing sufficiently small $a$, converges to zero as $T \rightarrow \infty$. In summary, $\Prob[\ED_{12}] = \so(1)$.

\noindent\textbf{Second step ($\Prob[\ED_{1\backslash2}] = \so(1)$):}
we have
\begin{align*}
\lefteqn{
 \Prob[\ED_{1\backslash2}]
} \nn
   & = \Prob\left[\sum_{i \in \actions} \nib{i} D(S_i\pst \Vert S_i\pmd) < (1-\epsilon) \log{T}, \nib{\imd} < \sqrt{T}, \sum_{i \in \actions} \hKL_i(\nib{i}) > \left(1 - \frac{\epsilon}{2}\right) \log{T} \right].
\end{align*}
Note that 
\begin{align*}
  \max_{1 \leq n \leq N} \hKL_i(n) = \max_{1 \leq n \leq N} \sum_{m=1}^n \log{\left(\frac{(S_i \pst)_{\Obsi^m}}{ (S_i \pmd)_{\Obsi^m} }\right)},
\end{align*}
 is the maximum of the sum of positive-mean random variables, and thus converges to is average  (c.f., Lemma 10.5 in \cite{bubeckthesis}). Namely, 
\begin{equation*}
 \lim_{N \rightarrow \infty} \max_{1 \leq n \leq N} \frac{\hKL_i(n)}{N} \rightarrow D(S_i\pst \Vert S_i\pmd)
\end{equation*}
almost surely. Therefore, 
\begin{equation*}
  \lim_{T \rightarrow \infty} \frac{ \max_{\{ \nib{i} \} \in \Natural^{N}, \sum_{i \in \actions} \nib{i} D(S_i\pst \Vert S_i\pmd) < (1-\epsilon) \log{T} } \sum_{i \in \actions} \hKL_i(\nib{i}) }{\log{T}} \rightarrow 1-\epsilon
\end{equation*}
almost surely. 
By using this fact and $1-\epsilon/2 > 1 - \epsilon$, we have
\begin{equation*}
  \Prob\left[ \max_{\{ \nib{i} \} \in \Natural^{N}, \sum_{i \in \actions} \nib{i} D(S_i\pst \Vert S_i\pmd) < (1-\epsilon) \log{T} }\sum_{i \in \actions} \hKL_i(\nib{i}) > \left(1 - \frac{\epsilon}{2}\right) \log{T} \right] = \so(1).
\end{equation*}
In summary, we obtain $\Prob\left[ \ED_{1\backslash2} \right] = \so(1)$.

\noindent\textbf{Last step:} we here have
\begin{align*}
\ED_1 & 
 = \left\{ \sum_{i \in \actions} \nib{i} D(S_i\pst \Vert S_i\pmd) < (1-\epsilon) \log{T} \right\} \cap \left\{\nib{\imd} < \sqrt{T} \right\}
 \nn
& \supseteq \left\{ \sum_{i \in \actions} \nib{i} D(S_i\pst \Vert S_i\pmd) + \frac{(1-\epsilon) \log{T}}{\sqrt{T}} \nib{\imd} < (1-\epsilon) \log{T} \right\}, 
\end{align*}
where we used the fact that $\{A<C\} \cap \{B<C\} \supseteq \{A+B<C\}$ for $A,B>0$ in the last line. 
Note that, by using the result of the previous steps, $\Prob[\ED_1] = \Prob[\ED_{12}] + \Prob[\ED_{1\backslash2}] = \so(1)$. By using the complementary of this fact,
\begin{align*}
\Prob\left[ \sum_{i \in \actions} \nib{i} D(S_i\pst \Vert S_i\pmd) + \frac{(1-\epsilon) \log{T}}{\sqrt{T}} \nib{\imd} \geq (1-\epsilon) \log{T} \right] \geq \Prob[\ED_1^c] = 1-\so(1).
\end{align*}
Using the Markov inequality yields 
\begin{equation}
 \Expect\left[ \sum_{i \in \actions} \nib{i} D(S_i\pst \Vert S_i\pmd) + \frac{(1-\epsilon) \log{T}}{\sqrt{T}} \nib{\imd} \right]
  \geq (1-\epsilon) (1-\so(1)) \log{T}. \label{ineq:lowerlast}
\end{equation}
Because $\Expect[\nib{\imd}]$ is subpolynomial as a function of $T$ due to the consistency, the second term in LHS of \eqref{ineq:lowerlast} is $\so(1)$ and thus negligible.
Lemma \ref{lem:drawlower} follows from the fact that \eqref{ineq:lowerlast} holds for sufficiently small $\epsilon$ and arbitrary $\pmd \in \interior{(\nonoptimal)}$.
\end{proof}

\begin{proof}[Proof of Theorem \ref{thm:regretlower}]

Assume that
there exists $\delta > 0$ and 
a sequence
$T_1<T_2<T_3<\cdots$
such that for all $t$
\begin{equation*}
\E[\regret(T_t)]
<
(1-\delta)\optcone{1}(p^*,\{p_i^*\})\log T_t\com
\end{equation*}
that is,
\begin{align*}
\sum_{i\neq 1}\frac{\Expect[N_i(T_t)]}{(1-\delta)\log T_t} (L_i - L_1)^\top \pst
< \optcone{1}(p^*,\{p_i^*\})\per
\end{align*}
From the definition of $\optcone{1}$,
there exists $q'_t \in \{q \in \closure(\nonoptimal): p_1^* = S_j q\}=:\mathcal{S}$ such that 
\begin{equation*}
\sum_{i \neq 1} \frac{\Expect[N_i(T_t)]}{(1-\de)\log T_t} D(p_i^* \Vert S_i q'_t) < 1\per
\end{equation*}
Since $\mathcal{S}$ is compact, there exists a subsequence
$t_0<t_1<\cdots$ such that
$\lim_{u\to\infty} q'_{t_u}=q'$
for some $q'\in\mathcal{S}$.
Therefore from the lower semicontinuity of the divergence we obtain
\begin{align*}
1
&\ge \sum_{i \neq 1} \liminf_{u\to\infty}
\frac{\Expect[N_i(T_t)]}{(1-\de)\log T_t} D(p_i \Vert S_i q'_{t_u})\nn
&\ge \sum_{i \neq 1} \liminf_{t\to\infty}
\frac{\Expect[N_i(T_t)]}{(1-\de)\log T_t}
D(p_i \Vert S_i q')\nn
&= \sum_{i} \liminf_{t\to\infty}
\frac{\Expect[N_i(T_t)]}{(1-\de)\log T_t}
D(p_i \Vert S_i q')
\com
\end{align*}
which contradicts Lemma \ref{lem:drawlower}.

\end{proof}

\section{The $N$-armed Bandit Problem as Partial Monitoring}
\label{sec:banditproof}

In Section \ref{sec:analysis}, we have introduced PM-DMED-Hinge, an asymptotically optimal algorithm for partial monitoring. In this appendix, we prove that this algorithm also has an optimal regret bound of the $N$-armed bandit problem when we run it on an $N$-armed bandit game represented as an instance of partial monitoring. 

In the $N$-armed bandit problem, the learner selects one of $N$ actions (arms) and receives a corresponding reward. 
This problem can be considered as a special case of partial monitoring in which the learner directly observes the loss matrix.
For example, three-armed Bernoulli bandit can be represented by the following loss and feedback matrices, and the strategy:
\begin{equation}
\label{ineq:bandit}
  L = H = \left(
    \begin{array}{cccccccc}
      2 & 1 & 2 & 1 & 2 & 1 & 2 & 1 \\
      2 & 2 & 1 & 1 & 2 & 2 & 1 & 1 \\
      2 & 2 & 2 & 2 & 1 & 1 & 1 & 1 \\
    \end{array}
  \right) \text{, and } \pst = \left(
    \begin{array}{c}
      (1-\mu_1)(1-\mu_2)(1-\mu_3) \\
      \mu_1(1-\mu_2)(1-\mu_3) \\
      (1-\mu_1)\mu_2(1-\mu_3) \\
      \mu_1 \mu_2(1-\mu_3) \\
      (1-\mu_1)(1-\mu_2)\mu_3 \\
      \mu_1(1-\mu_2)\mu_3 \\
      (1-\mu_1)\mu_2\mu_3 \\
      \mu_1 \mu_2 \mu_3 \\
    \end{array}
  \right),
\end{equation}
where $\mu_1, \mu_2$, and $\mu_3$ are the expected rewards of the actions.
Signals $1$ and $2$ correspond to the rewards of $1$ and $0$ generated by the selected arm, respectively.
More generally, $N$-armed Bernoulli bandit is represented as an instance of partial monitoring in which the loss and feedback matrices are the same $N \times 2^N$ matrix  
\begin{equation*}
 l_{i,j} = h_{i,j} = \Ind[ (j-1 \mymod 2^i) < 2^{i-1}] + 1,
\end{equation*}
 where mod denotes the modulo operation. This problem is associated with $N$ parameters $\mu_1,\mu_2,\dots,\mu_N$ that correspond to the expected rewards of the actions. For the ease of analysis, we assume $\{\mu_i\}$ are in $(0,1)$ and different from each other. Without loss of generality, we assume $1>\mu_1 > \mu_2 > \dots > \mu_N>0$, and thus action $1$ is the optimal action. 
The opponent's strategy is 
\begin{equation*}
 \pst_j = \prod_{i \in [N]} (\mu_i + (1-2 \mu_i) \Ind[ (j-1 \mymod 2^i) < 2^{i-1}] ) \per
\end{equation*} 
Note that $\mu_i = (S_{i} \pst)_1$.

\begin{proof}[Proof of Corollary \ref{cor:bandit}]
In the following, we prove that the regularity conditions in Theorem \ref{thm_main} are always satisfied in the case of the $N$-armed bandit. During the proof we also show that $\optcone{1}(p^*,\{p_i^*\})$ is equal to the optimal constant factor of Lai and Robbins \cite{LaiRobbins1985}.

Because signal $1$ corresponds to the reward of $1$, we can define $\hatmu_i(q) = (S_i q)_1$, and thus
\begin{equation*}
  \cell{i} = \{q \in \simplex: \forall_{i' \neq i}\ \hatmu_i(q) \ge \hatmu_{i'}(q) \}.
\end{equation*}
First, we show the uniqueness of $\optset{1}{p,\{p_i,\de_i\}}$ at $p=p^*,\{p_i\}=S_i p^*,\,\de_i=0$.
It is easy to check 
\begin{equation*}
D(p^*_i\Vert S_i q) = d(\hatmu_i(\pst) \Vert \hatmu_i(q)) = d(\mu_i \Vert \hatmu_i(q)),
\end{equation*}
where $d(a \Vert b)$ is the KL divergence between two Bernoulli distributions with parameters $a$ and $b$. Then
\begin{align}
  \admq{1}{\{p_i^*\}} & = 
\left\{
\{r_i\}_{i\neq 1}\in [0,\infty)^{N-1}:
\inf_{q\in \closure(\cP_{1}^c): p_i^* = S_1 q }
\sum_{i} r_i D(p_i^*\Vert S_i q) \ge 1
\right\} \nn
& = \left\{
\{r_i\}_{i\neq 1}\in [0,\infty)^{N-1}:
\inf_{q\in \closure(\cP_{1}^c): \mu_1=\hatmu_1(q) }
\sum_{i} r_i D(p_i^*\Vert S_i q) \ge 1
\right\} \nn
  & = \left\{\{r_i\}_{i \neq 1}: r_i \ge \frac{1}{d(\mu_i \Vert \mu_1)} \right\}, \label{ineq:admqbandit}
\end{align}
where the last inequality follows from the fact that 
\begin{equation*}
 \{q \in \closure(\cP_{1}^c): \hatmu_1(q)=\mu_1 \} = \{q \in \simplex: \hatmu_1(q)=\mu_1, \exists_{i \neq 1} \hatmu_i(q) \ge \mu_1\}.
\end{equation*}
  By Eq.\,\eqref{ineq:admqbandit}, the regret minimizing solution is
\begin{align*}
  \optcone{1}(p^*,\{p_i^*\}) & = \sum_{i \neq 1} \frac{\Delta_i}{d(\mu_i \Vert \mu_1)},
\end{align*}
and 
\begin{align*}
  \optset{1}{p^*,\{p_i^*\}} & = \left\{\{r_i\}_{i \neq 1}: r_i = \frac{1}{d(\mu_i \Vert \mu_1)} \right\},
\end{align*}
which is unique.

Second, we show that $\closure(\interior(\cP_1^c)\cap \setd_{\de}) = \closure(\closure(\cP_1^c)\cap \setd_{\de})$ 
for sufficiently small $\delta \ge 0$.
Note that,
\begin{equation*}
 \closure(\cP_1^c) \cap \setd_{\de} = \{q \in \simplex: \exists_{i' \neq 1}\ \hatmu_1(q) \le \hatmu_{i'}(q), d(\mu_1 \Vert \hatmu_{1}(q)) \le \delta\}
\end{equation*}
and
\begin{equation*}
 \interior(\cP_1^c) \cap \setd_{\de} = \{q \in \simplex: \exists_{i' \neq 1}\ \hatmu_1(q) < \hatmu_{i'}(q), d(\mu_1 \Vert \hatmu_{1}(q)) \le \delta\}.
\end{equation*}
To prove
\begin{equation}
 \closure(\cP_1^c)\cap \setd_{\de} \subset \closure(\interior(\cP_1^c)\cap \setd_{\de}), \label{clint}
\end{equation}
 it suffices to show that, an open ball centered at any position in 
\begin{equation*}
  \{q \in \simplex: \exists_{i' \neq 1}\ \hatmu_1(q) = \hatmu_{i'}(q), d(\mu_1 \Vert \hatmu_{1}(q)) \le \delta\} \,\supset\, (\closure(\cP_1^c) \cap \setd_{\de} ) \setminus (\interior(\cP_1^c) \cap \setd_{\de})
\end{equation*}
contains a point in $\interior(\cP_1^c) \cap \setd_{\de}$. 
This holds because we can make a slight move towards the direction of increasing $\hatmu_{i'}$: we can always find $q'$ in an open ball centered at $q$ such that $\hatmu_{i'}(q') > \hatmu_{i'}(q)$ and $\hatmu_{1}(q') = \hatmu_{1}(q)$ because of (i) the fact that there always exists $q \in \simplex$ such that $\{q \in \simplex, \forall_{i \in [N]} \hatmu_i(q) = \mu_i\}$ for arbitrary $\{\mu_i\} \in (0,1)^N$ and (ii) the continuity of the $\hatmu_i$ operator. Therefore, any open ball centered at $q \in \closure(\cP_1^c)\cap \setd_{\de}$ contains an element of $\interior(\cP_1^c)\cap \setd_{\de}$, by which we obtain \eqref{clint}. By using \eqref{clint}, we have
\begin{equation}
 \closure(\closure(\cP_1^c)\cap \setd_{\de}) \subset \closure(\closure(\interior(\cP_1^c)\cap \setd_{\de})) = \closure(\interior(\cP_1^c)\cap \setd_{\de}), \label{clcleq}
\end{equation}
where we used the fact that $\closure(\closure(X)) = \closure(X)$.
Combining \eqref{clcleq} with the fact that $\closure(\closure(\cP_1^c)\cap \setd_{\de}) \supset \closure(\interior(\cP_1^c)\cap \setd_{\de})$ yields $\closure(\closure(\cP_1^c)\cap \setd_{\de}) = \closure(\interior(\cP_1^c)\cap \setd_{\de})$.

Therefore, in the $N$-armed Bernoulli bandit problem, the regularity conditions are always satisfied and $\optcone{1}(p^*,\{p_i^*\})$ matches the optimal coefficient of the logarithmic regret bound. From Theorem \ref{thm_main}, if we run PM-DMED-Hinge in this game, its expected regret is asymptotically optimal in view of the $N$-armed bandit problem.
\end{proof}

\subsection{Experiment}
\label{subsec:banditexperiment}

\begin{figure}[t]
\begin{center}
\includegraphics[scale=0.5]{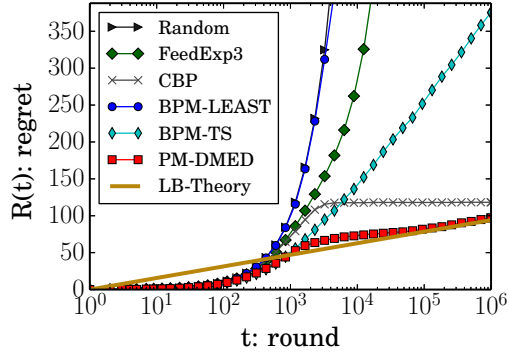}
\caption{Regret-round semilog plots of algorithms. The regrets are averaged over 100 runs. LB-Theory is the asymptotic regret lower bound of Lai and Robbins \cite{LaiRobbins1985}.}
\label{fig:bandit}
\end{center}
\end{figure}

We also assessed the performance of PM-DMED and other algorithms in solving the three-armed Bernoulli bandit game defined by \eqref{ineq:bandit} with parameters $\mu_1=0.4$, $\mu_2=0.3$, and $\mu_3=0.2$. The settings of the algorithms are the same as that of the main paper.
The results of simulations are shown in Figure \ref{fig:bandit}. 
LB-Theory is the regret lower bound of Lai and Robbins \cite{LaiRobbins1985},
that is, $\sum_{i \neq 1} \frac{\Delta_i \log{t}}{d(\mu_i \Vert \mu_1)}$.
The slope of PM-DMED quickly approaches that of LB-Theory, which is empirical evidence that PM-DMED has optimal performance in $N$-armed bandits.

\section{Optimality of PM-DMED-Hinge}\label{append_main}
In this appendix we prove Theorem \ref{thm_main}.
First we define distances among distributions.
For distributions $p_i,p_i'\in \mathcal{P}_{A}$ of symbols
we use the total variation distance
\begin{align*}
\kyoria{p_i-p_i'}=\frac{1}{2}\sum_{a=1}^A |(p_i)_a-(p_i')_a|\per\n
\end{align*}
For distributions $p,p'\in\simplex$ of outcomes, we identify
$p$ with the set $\{p':\forall i,\,S_ip=S_ip'\}$ and define
\begin{align*}
\kyorim{p-p'}=\max_{i}\kyoria{S_ip-S_ip'}.
\end{align*}
For $\mathcal{Q}\subset \simplex$ we define
\begin{align*}
\kyorim{p-\mathcal{Q}}=\inf_{p'\in\mathcal{Q}}\kyorim{p-p'}\per\n
\end{align*}

In the following, we use Pinsker's inequality given below many times.
\begin{align*}
D(p_i \Vert q_i)\ge 2\kyoria{p_i-q_i}^2\per\n
\end{align*}

Let
\begin{align*}
\rho_{i,L} &=\sup_{\lambda>0}\frac{1}{\lambda}
\min_{x\in \cP_{i,\lambda}}\kyorim{x-\cP_i^c}\com\nn
\nu_{i,L}&=\sup_{\lambda>0}\frac{1}{\lambda}
\max_{x\in \cP_{i,\lambda}^c}\kyorim{x-\cP_i^c}\per\nn
\end{align*}
Note that these two constants are positive from the global observability.

\subsection{Properties of regret lower bound}
In this section, we give Lemma \ref{lem_opt}
and Theorem \ref{thm_conti} that are about the functions
$\optc{j}{p,\,\{p_i,\de_i\}}$ and $\optset{j}{p,\{p_i,\de_i\}}$.
In the following, we always consider these functions
on $p\in \simplex$, $p_i\in\{S_ip:\mathrm{supp}(p)\subset \mathrm{supp}(p^*)\}$
and $\de_i\ge 0$, where $\mathrm{supp}(\cdot)$ denotes the support of the distribution.

We define 
\begin{equation*}
\Lmax = \max_{i',j'}l_{i',j'}\per
\end{equation*}
\begin{lemma}\label{lem_opt}
Let $p\in \cP_{j,\al}$ and $\{p_i,\de_i\}$ be satisfying
$\kyoria{p_i-S_i p}\le \norm{S}\al\rho_{j,L}/2$
and $\de_i\le (\norm{S}\al\rho_{j,L})^2/4$
for all $i$.
Then
\begin{align}
\optc{j}{p,\,\{p_i,\de_i\}}\le
\frac{4N\Lmax}{(\norm{S}\al\rho_{j,L})^2}\per\label{c_bound}
\end{align}
Furtheremore, $\adm{j}{\{p_i,\de_i\}}$ is nonempty and
\begin{align*}
\optset{j}{p,\{p_i,\de_i\}}\subset \left[
0,\,\frac{4N\Lmax}{(\norm{S}\rho_{j,L})^2\al^3}
\right]^{N-1}\per 
\end{align*}
\end{lemma}
\begin{proof}[Proof of Lemma \ref{lem_opt}]
Since $\kyorim{p-\cP_1^c}\ge \al\rho_{j,L}$,
there exists $i = i(q)$ for any $q\in\cP_1^c$ such that
\begin{align*}
\kyoria{S_{i}q-S_{i}p}\ge
\norm{S}\al\rho_{j,L}\per\n
\end{align*}
For this $i$ we have
\begin{align*}
D(p_i\Vert S_i q)-\de_i
&\ge
2\kyoria{p_i-S_i q}^2-\de_i\nn
&\ge
2(\kyoria{S_iq-S_i p}-\kyoria{p_i-S_i p})_+^2-\de_i\nn
&\ge
(\norm{S}\al\rho_{j,L})^2/2-\de_i\nn
&\ge
(\norm{S}\al\rho_{j,L})^2/4\per\n
\end{align*}
Thus, by letting $r_i=4/(\norm{S}\al \rho_{j,L}\al)^2$ for all $i\neq j$
we have
\begin{align*}
\{r_i\}_{i\neq j}\in \adm{j}{\{p_i,\de_i\}}\com\n
\end{align*}
which implies \eqref{c_bound}.
On the other hand it holds for any $\{r_i^*\}_{i\neq j}\in \optset{j}{p,\,\{p_i,\de_i\}}$
from $p\in \cP_{j,\al}$
that
\begin{align*}
\optc{j}{p,\,\{p_i,\de_i\}}=
\sum_{i\neq j}r_i^*L_i^\top p
\ge \max_{i\neq j}r_i^* \al\n
\end{align*}
and therefore we have
\begin{align*}
\max_{i\neq j}r_i^*\le
\frac{4N\Lmax}{(\norm{S}\rho_{j,L})^2\al^3}\per\n
\end{align*}
\end{proof}

\begin{theorem}\label{thm_conti}
Assume that the regularity conditions in Theorem \ref{thm_main} hold.
Then the point-to-set map $\optset{1}{p,\{p_i,\de_i\}}$ is
(i) nonempty near $p=p^*,p_i=S_i p^*,\,\de_i=0$
and (ii) continuous at
$p=p^*,p_i=S_i p^*,\,\de_i=0$.
\end{theorem}
See Hogan \cite{hogan}
for definitions of terms such as continuity of point-to-set maps.

\begin{proof}[Proof of Theorem \ref{thm_conti}]
Define
\begin{align*}
\admb{1}{\{p_i,\de_i\}}
&=
\left\{
\{r_i\}_{i\neq 1}\in [0,\xi]^{N-1}:
\inf_{q\in \closure(\cP_{1}^c):
D(p_1\Vert S_1 q)\le \de_1}
\sum_{i\neq 1} r_i(D(p_i\Vert S_i q)-\de_i)_+\ge 1
\right\}
\end{align*}
for
\begin{align*}
\xi=\frac{4N\Lmax}{(\norm{S}\rho_{1,L})^2(\max_{i\neq 1}L_i^\top p^*-L_1^\top p^*)^3}\per\n
\end{align*}
Note that $p^* \in \cP_{1,\al}$ for $\al \le \max_{i\neq 1}L_i^\top p^*-L_1^\top p^*$.
From Lemma \ref{lem_opt}, near $p=p^*,p_i=S_i p^*,\,\de_i=0$,
\begin{align*}
\admb{1}{\{p_i,\de_i\}}
&\supset
\optset{1}{p,\{p_i,\de_i\}}
\end{align*}
and
\begin{align*}
\optc{1}{p,\{p_i,\de_i\}}&=
\inf_{\{r_i\}_{i\neq 1}\in \admb{1}{\{p_i,\de_i\}}}
\sum_{i\neq 1} r_i (L_i-L_1)^\top p\n
\end{align*}
hold.
Since the function
\begin{align*}
\sum_{i} r_i(D(p_i\Vert S_i q)-\de_i)_+
\end{align*}
is continuous in $\{r_i\}$,
$\admb{1}{\{p_i,\de_i\}}$ is a closed set and therefore
$\optset{1}{p,\{p_i,\de_i\}}$ is nonempty
near $p=p^*,p_i=S_i p^*,\,\de_i=0$.

From the continuity of $D(p_i\Vert S_i q)$ at any $q$
such that $D(p_i\Vert S_i q)<\infty$,
we have
\begin{align*}
\inf_{q\in \closure(\cP_{1}^c)\cap \setd_{\de_1}}
\sum_{i\neq 1} r_i(D(p_i\Vert S_i q)-\de_i)_+
&=
\inf_{q\in \closure(\closure(\cP_{1}^c)\cap \setd_{\de_1})}
\sum_{i\neq 1} r_i(D(p_i\Vert S_i q)-\de_i)_+\nn
&=
\inf_{q\in \closure(\interior(\cP_{1}^c)\cap \setd_{\de_1})}
\sum_{i\neq 1} r_i(D(p_i\Vert S_i q)-\de_i)_+\nn
&=
\inf_{q\in \interior(\cP_{1}^c)\cap \setd_{\de_1}}
\sum_{i\neq 1} r_i(D(p_i\Vert S_i q)-\de_i)_+\per\n
\end{align*}
Thus, we have
\begin{align}
\admb{1}{\{p_i,\de_i\}}
&=
\left\{
\{r_i\}_{i\neq 1}\in [0,\xi]^{N-1}:
\inf_{q\in \interior(\cP_{1}^c):
D(p_1\Vert S_1 q)\le \de_1}
\sum_{i} r_i(D(p_i\Vert S_i q)-\de_i)_+\ge 1
\right\}\per\label{new_feasible}
\end{align}

Since the objective function $\sum_{i\neq j} r_i (L_i-L_j)^\top p$
is continuous in $\{r_i\}$ and $p$,
and \eqref{new_feasible} is compact,
now it suffices to show that
\eqref{new_feasible} is continuous
in $\{p_i,\de_i\}$ at $\{S_ip^*,\,0\}$
to prove the theorem from \cite[Corollary 8.1]{hogan}.

First we show that $\admb{1}{\{p_i,\de_i\}}$ is closed
at $\{S_ip^*,\,0\}$.
Consider $\{r_i^{(m)}\}_{i\neq 1}\in \admb{1}{\{p_i^{(m)},\de_i^{(m)}\}}$
for a sequence $\{p_i^{(m)},\de_i^{(m)}\}_i$ which converges to
$\{S_i p^*,\,0\}_i$ as $m\to\infty$.
We show that $\{r_i\}_{i\neq 1}\in \admb{1}{\{S_i p^*,0\}}$
if $r_i^{(m)}\to r_i$ as $m\to\infty$.

Take an arbitrary $q\in \interior(\cP_1^c)$ such that $D(S_1p^*\Vert S_1q)=0$.
Since $\kyoria{S_1p^*-p_1^{(m)}}\to 0$
and $p_1\in \{S_1p:\mathrm{supp}(p)\subset \mathrm{supp}(p^*)\}$,
there exists $\tp^{(m)}$ such that
$p_1^{(m)}=S_1\tp^{(m)}$ and $\kyorim{p^*-\tp^{(m)}}\to 0$.
Thus, from $q\in \interior(\cP_1^c)$,
it holds for sufficiently large $m$ that
$q^{(m)}=q-p^*+\tp^{(m)}\in \interior(\cP_1^c)$.
For this $q^{(m)}$ we have
\begin{align*}
D(p_1^{(m)}\Vert S_1q^{(m)})
&\le
D(S_1\tp^{(m)}\Vert S_1(q-p^*+\tp^{(m)}))=0\le \de_1\per\n
\end{align*}
that is, $q^{(m)}\in \interior(\cP_1^c)\cap \setd_{\de_1}$.
Therefore for sufficiently large $m$
we have
\begin{align*}
\sum_{i} r_i^{(m)}(D(p_i\Vert S_i q^{(m)})-\de_i^{(m)})_+\ge 0\per\n
\end{align*}
and, letting $m\to\infty$,
\begin{align*}
\sum_{i} r_iD(p_i\Vert S_i q)\ge 0\per\n
\end{align*}
This means that $\{r_i\}_{i\neq 1}\in \admb{1}{\{p_i,\de_i\}}$,
that is, $\admb{1}{\{p_i,\de_i\}}$ is closed
at $\{S_ip^*,\,0\}$.

Next we show that
$\admb{1}{\{p_i,\de_i\}}$ is open
at $\{S_ip^*,\,0\}$.
Consider $\{r_i\}_{i\neq 1}\in \admb{1}{\{S_ip^*,0\}}$
and a sequence $\{p_i^{(m)},\de_i^{(m)}\}_i$ which converges to
$\{S_i p^*,\,0\}_i$ as $m\to\infty$.
We show that there exists a sequence $\{r_i^{(m)}\}_{i\neq 1}\in \admb{1}{\{p_i^{(m)},\de_i^{(m)}\}}$
such that $r_i^{(m)}\to r_i$.

Consider the optimal value function
\begin{align}
v(\{p_i^{(m)},\de_i^{(m)}\})&=\inf_{q\in \closure(\cP_1^c) \cap \setd_{\de_1}}
\sum_{i} r_i(D(p_i^{(m)}\Vert S_i q)-\de_i^{(m)})_+\per\label{semiconti}
\end{align}
Since the feasible region of \eqref{semiconti} is closed at $p_i=S_ip^*,\,\de_i=0$ and
the objective function of \eqref{semiconti} is lower semicontinuous
in $q,\,\{p_i,\,\de_i\}$
we see that $v(\{p_i^{(m)},\de_i^{(m)}\})$ is lower semicontinuous from \cite[Theorem 2]{hogan}.
Therefore, for any $\ep>0$ there exists $m_0>0$ such that
for all $m\ge m_0$
\begin{align*}
v(\{p_i^{(m)},\de_i^{(m)}\})\ge (1-\ep)v(\{S_ip^*, 0\})\ge 1\n
\end{align*}
since $v(\{S_ip^*, 0\})\ge 1$ from $r_i\in \admb{1}{\{S_ip^*,0\}}$.
Thus, by letting $r_i^{(m)}:=r_i/(1-\ep)$
we have
\begin{align*}
\inf_{v\in \closure(\cP_1^c) \cap \setd_{\de_1}}
\sum_{i} r_i^{(m)}(D(p_i^{(m)}\Vert S_i q^{(m)})-\de_i^{(m)})_+\ge 1\com\n
\end{align*}
that is, $\{r_i^{(m)}\}_{i\neq 1}\in \admb{1}{\{p_i^{(m)},\de_i^{(m)}\}}$.
\end{proof}

\subsection{Regret analysis of PM-DMED-Hinge}
Let $\phat_{i,n} \in [0,1]^A$
be the empirical distribution of the symbols from the action $i$
when the action $i$ is selected $n$ times.
Then we have $\phat_i(t)=\phat_{i,N_i(t)}$.
Let $P_{i,n_i}(u)=\Pr[D(\phat_{i,n_i}\Vert S_i p^*)\ge u]$.
Then, from the large deviation bound on discrete distributions
(Theorem 11.2.1 in Cover and Thomas \cite{CoverThomasSnd}), we have
\begin{align}
P_{i,n_i}(u)\le (n_i+1)^{A}\e^{-n_i u}\per \label{ldpdisc}
\end{align}

We also define
\begin{equation*}
\sets(\{p_i,n_i\})=\{i\subset [N]: D(p_i\Vert S_i p^*)-f(n_i)> 0\}.
\end{equation*}

For
\begin{align}
0<\de\le \normss\kyorim{p^*-\cP_1^c}^2/8
\label{assump_delta}
\end{align}
define events
\begin{align}
\cA(t)&=\{\phat(t)\in \cP_{1}\}\nn
\cA'(t)&=\{\phat(t)\in \cP_{1,\al(t)}\}\nn
\variation(t)&=\bigcap_i\{\kyoria{\phat_i(t)-S_i p^*}\le \sqrt{\de}\}\nn
\trueD(t)&=\{\ihat(t)\notin \sets(\{\phat_i(t),N_i(t)\}),\,\sets(\{\phat_i(t),N_i(t)\})\neq \emptyset\}\nn
&=\left\{
D(\phat_{\ihat(t)}(t)\Vert S_{\ihat(t)}p^*)\le f(N_{\ihat(t)}(t)),\,\bigcup_{i}\{D(\phat_i(t)\Vert S_ip^*)> f(N_i(t))\}\right\}\nn
\empD(t)&=\bigcap_i\{D(\phat_i(t)\Vert S_i\phat(t))\le f(N_i(t))\}\nn
\other(t)&=\bigg\{
\max_if(N_i(t))\le \min\left\{2\de,\,(\norm{S}\rho_{1,L}\al(t))^2/4\right\},
\label{prop_e1}\\
&\qquad
\min_i N_i(t)\ge \max\{c\sqrt{\log t},\,(\log \log T)^{1/3}\},\,
2\nu_{1,L}\al(t)\le \kyorim{p^*-\cP_1^c}
\bigg\}\com \nonumber 
\end{align}
where we write $\{\mathcal{T},\,\mathcal{U}\}$ instead of
$\{\mathcal{T}\,\cap \mathcal{U}\}$ for events $\mathcal{T}$ and $\mathcal{U}$.

\begin{proof}[Proof of theorem \ref{thm_main}]
Since $\cA'(t)\subset\cA(t)$, the whole sample space is covered by
\begin{align}
\lefteqn{
\{\cA'(t),\variation(t)\}\scup\{\cA'(t),\variation^c(t)\}\scup
\{\cA(t),(\cA'(t))^c\}\scup \{\cA^c(t),\trueD(t)\}\scup
\{\cA^c(t),\trueD^c(t)\}
}\nn
&\subset
\{\cA'(t),\variation(t),\empD(t),\other(t)\}\scup\{\cA'(t),\variation^c(t),\empD(t),\other(t)\}\scup
\{\cA(t),(\cA'(t))^c,\empD(t),\other(t)\}\scup \{\cA^c(t),\trueD(t)\}\nn
&\quad\scup \{\cA^c(t),\trueD^c(t),\empD(t),\other(t)\}
\scup \empD^c(t)\scup \other^c(t)\per\label{space}
\end{align}

Let $\jt{i}{t}$ denote the event that
 action $i$ is newly added into the list $L_N$
at the $t$-th round
and
$\jtp{i}{t}\subset \jt{i}{t}$ denote the event that
$\jt{i}{t}$ occurred by Step 6 of Algorithm \ref{alg_pmdmedhinge}.
Note that if $\{\cA'(t),\,\empD(t),\,\other(t)\}$ occurred then
$\jt{i}{t}$ is equivalent to $\jtp{i}{t}$.
Combining this fact with \eqref{space} we can bound the regret as
\begin{align*}
\regret(T)
&\le
\sum_{i\neq 1}\De_i\sum_{t=1}^T
\idx{\jt{i}{t}}+N\nn
&\le
\sum_{i\neq 1}\De_i\sum_{t=1}^T
\bigg(
\idx{\jtp{i}{t},\cA'(t),\,\variation(t),\,\empD(t),\,\other(t)}
+
\idx{\jtp{i}{t},\cA(t),\,\variation^c(t),\,\empD(t),\,\other(t)}\nn
&\qquad\qquad
+\idx{\jt{i}{t},\cA(t),(\cA'(t))^c,\empD(t),\other(t)}
+\idx{\jt{i}{t},\cA^c(t),\trueD^c(t),\empD(t),\other(t)}\nn
&\qquad\qquad+\idx{\jt{i}{t},\empD^c(t)\scup \other^c(t)}\bigg)
+
\left(
\sum_{i\neq 1}\De_i
\right)
\sum_{t=1}^T\idx{\cA^c(t),\,\trueD(t)}+N\per\n
\end{align*}
The following Lemmas \ref{lem_main}--\ref{lem_other} bound the
expectation of each term and complete the proof.
\end{proof}

\begin{lemma}\label{lem_main}
Let $\{r_i^*\}_{i\neq 1}$ be the unique member of
$\optset{j}{p^*,\{S_i p^*,0\}}$.
Then there exists $\ep_{\de}>0$ such that
$\lim_{\de\to 0}\ep_{\de}=0$ and
for all $i\neq 1$
\begin{align*}
\sum_{t=1}^T
\idx{\jtp{i}{t},\,\cA'(t),\,\variation(t),\,\empD(t),\,\other(t)}
\le
(1+\ep_{\de})r_i^* \log T+1\per\n
\end{align*}
\end{lemma}

\begin{lemma}\label{lem_sublog}
\begin{align*}
\E\left[
\sum_{t=1}^{T}
\idx{\jtp{i}{t},\,\cA'(t),\,\variation^c(t),\,\empD(t),\,\other(t)}
\right]
&=\so(\log T)\per\n
\end{align*}
\end{lemma}

\begin{lemma}\label{lem_kantigai}
\begin{align*}
\E\left[
\sum_{t=1}^{T}
\idx{\cA^c(t),\,\trueD(t)}
\right]
&=\lo(1)\per\n
\end{align*}
\end{lemma}

\begin{lemma}\label{lem_marginal}
\begin{align*}
\E\left[\sum_{t=1}^{T}\idx{\jt{i}{t},\,\cA(t),\,(\cA'(t))^c,\,\empD(t),\,\other(t)}
\right]
&=\lo(1)\per\n
\end{align*}
\end{lemma}

\begin{lemma}\label{lem_trueD}
\begin{align*}
\E\left[
\sum_{t=1}^{T}
\idx{\jt{i}{t},\,\cA^c(t),\,\trueD^c(t),\,\empD(t),\,\other(t)}
\right]
&=\lo(1)\per\n
\end{align*}
\end{lemma}

\begin{lemma}\label{lem_empD}
\begin{align*}
\E\left[\sum_{t=1}^T\idx{\jt{i}{t},\,\empD^c(t)}
\right]
&=\lo(1)\per\n
\end{align*}
\end{lemma}

\begin{lemma}\label{lem_other}
\begin{align*}
\sum_{t=1}^T\idx{\jt{i}{t},\,\other^c(t)}
&=\so(\log T)\per\n
\end{align*}
\end{lemma}

\begin{proof}[Proof of Lemma \ref{lem_main}]
From $\empD(t)$ we have
\begin{align}
\sum_{i}N_i(t) (D(\phat_i(t)\Vert S_i \phat(t))-f(N_i(t)))_+=0\per\label{contradict}
\end{align}
Here assume that
$\kyorim{\phat(t)-p^*}> 2\sqrt{\de}$.
Then
\begin{align*}
\max_i D(\phat_i(t)\Vert S_i \phat(t))
&\ge
2\max_i \kyoria{\phat_i(t)-S_i\phat(t)}^2
\since{by Pinsker's inequality}
\nn
&\ge
2\max_i (\kyoria{S_ip^*-S_i\phat(t)}-\kyoria{S_ip^*-\phat_i(t)})_+^2
\nn
&\ge
2\max_i (\kyoria{S_ip^*-S_i\phat(t)}-\sqrt{\de})_+^2
\since{by definition of $\variation(t)$}
\nn
&>
2\de\nn
&\ge
f(N_i(t))
\com
\since{by definition of $\other(t)$}
\n
\end{align*}
which contradicts
\eqref{contradict}
and we obtain
$\kyorim{\phat(t)-p^*}\le 2\sqrt{\de}$.
Furthermore,
from $\variation(t)$ and $\other(t)$
we have
\begin{align*}
\bigcap_i\{\kyoria{\phat_i(t)-S_i p^*} \le \sqrt{\de}\}
\mbox{ and }
\bigcap_i\{f(N_i(t))\le2\de\}\com\n
\end{align*}
respectively.
Since
$\optset{1}{p,\{p_i,\de_i\}}$
is continuous
at $p=p^*,\,p_i=S_i p^*,\,\de_i=0$
from Theorem \ref{thm_conti},
$r_i\le (1+\ep_{\de})r_i^*$
for all $\{r_i\}_{i\neq 1}\in \optset{\ihat(t)}{\phat(t),\{\phat_i(t),f(N_i(t))\}}$
where $r_i^*$ is the unique member of $\optset{1}{p^*,\{S_ip^*,0\}}$
and we used the fact that
$\cA'(t)$ implies $\ihat(t)=1$.

We complete the proof by
\begin{align*}
\lefteqn{
\sum_{t=1}^T
\idx{\jtp{i}{t},\,\cA'(t),\,\variation(t),\,\empD(t),\,\other(t)}
}\nn
&=
\sum_{n=1}^T
\idx{\bigcup_{t=1}^T\{\jtp{i}{t},\,\cA'(t),\,\variation(t),\,\empD(t),\,\other(t),N_i(t)=n\}}\nn
&\le
\sum_{n=1}^T
\idx{\bigcup_{t=1}^T\{n/\log t\le (1+\ep_{\de})r_i^*\}}\nn
&\le
(1+\ep_{\de})r_i^*\log T+1\per\n
\end{align*}
\end{proof}

\begin{proof}[Proof of Lemma \ref{lem_sublog}]
First, we obtain from $\empD(t)$ and $\other(t)$
that $f(N_i(t))\le (\norm{S}\rho_{1,L}\al(t))^2/4$
and
\begin{align*}
\kyoria{\phat_i(t)-S_i \phat(t)}
&\le
\sqrt{D(\phat_i(t)\Vert S_i \phat(t))/2}\nn
&\le
\sqrt{f(N_i(t))/2}\nn
&\le
\norm{S}\rho_{1,L}\al(t)/\sqrt{8}\per\n
\end{align*}
Therefore, from Lemma \ref{lem_opt}, it holds for
any $\{r_i^*\}_{i\neq 1}\in \optset{j}{\phat(t),\,\{\phat_i(t),f(N_i(t))\}}$
that
\begin{align*}
r_i^*
&\le
\frac{4N\Lmax}{(\norm{S}\rho_{1,L})^2(\al(t))^3}\nn
&\le
\frac{4N\Lmax}{(\norm{S}\rho_{1,L})^2(\al(T))^3}\per\n
\end{align*}

Now we have
\begin{align}
\lefteqn{
\E\left[
\sum_{t=1}^{\infty}
\idx{\jtp{i}{t},\,\cA'(t),\,\variation^c(t),\,\empD(t),\,\other(t)}
\right]
}\nn
&\le
\E\left[
\sum_{t=1}^{\infty}
\idx{
\frac{N_i(t)}{\log T}<
\frac{4N\Lmax}{(\norm{S}\rho_{1,L})^2(\al(T))^3},\,
\variation^c(t),\,\other(t)}
\right]\nn
&\le
\left(\frac{4N\Lmax\log T}{(\norm{S}\rho_{1,L})^2(\al(T))^3}+1\right)
\Pr\left[
\bigcup_{t=1}^T \{\variation^c(t),\,\other(t)\}
\right].\label{bound_bc}
\end{align}
Here, note that
\begin{align*}
\variation^c(t)
&\subset
\bigcup_{i}\{\kyoria{\phat_i(t)-S_i p^*}\ge \sqrt{\de}\}\nn
&\subset
\bigcup_{i}\{D(\phat_i(t)\Vert S_i p^*)\ge 2\de\}\per\n
\end{align*}
Since $N_i(t)\ge (\log \log T)^{1/3}$ holds under event $\other(t)$,
we can bound the probability in \eqref{bound_bc} as
\begin{align*}
\lefteqn{
\Pr\left[
\bigcup_{t=1}^T \{\variation^c(t),\,\other(t)\}
\right]
}\nn
&\le
\sum_{i}\sum_{n_i=(\log\log T)^{1/3}}^{\infty}\Pr[D(\phat_{i,n_i}\Vert S_i p^*)\ge 2\de]\nn
&\le
N\sum_{n=(\log\log T)^{1/3}}^{\infty}(n+1)^A\e^{-2n\de} \since{by \eqref{ldpdisc}}\nn
&=
\e^{-\Theta((\log\log T)^{1/3})}\n
\end{align*}
and combining this with \eqref{bound_bc} we have
\begin{align*}
\lefteqn{
\E\left[
\sum_{t=1}^{\infty}
\idx{\jtp{i}{t},\,\cA'(t),\,\variation^c(t),\,\empD(t),\,\other(t)}
\right]
}\nn
&\le
\lo\left((\log T)(\log \log T)^3\right)\e^{-\Theta((\log \log T)^{1/3})}\nn
&=\so(\log T)\per\n
\end{align*}
\end{proof}

\begin{proof}[Proof of Lemma \ref{lem_kantigai}]
Let $\mathcal{G}\in 2^{[N]}\setminus \emptyset$
and
$\{n_i\}_{i\in \mathcal{G}}\in\mathbb{N}^{|\mathcal{G}|}$
be arbitrary.
Consider the case that 
\begin{align}
\sum_{i\in \mathcal{G}} n_i(D(\phat_{i,n_i}\Vert S_i p^*)-f(n_i))_+<x\per\label{cond_phat_s}
\end{align}
for some $x>0$.
Then under events
$t\ge \e^x,\,\bigcap_{i\in \mathcal{G}}\{N_i(t)=n_i\},\,\cA^c(t),\,\trueD(t)$
and 
$\sets(\{\phat_i(t),\,f(n_i)\})=\mathcal{G}$
we have
\begin{multline*}
\min_{p\in\cP_{\ihat(t)}^c:D(\phat_{\ihat(t)}(t)\Vert S_{\ihat(t)} p)\le f(N_{\ihat(t)}(t))}
\sum_{i} N_i(t)(D(\phat_i(t)\Vert S_i p)-f(N_i(t)))_+
 \nn
\le
\sum_{i} n_i(D(\phat_i(t)\Vert S_i p^*)-f(n_i))_+
 < x \le \log t\com
\end{multline*}
which implies that the condition \eqref{cond_tansaku} is satisfied.
On the other hand from \eqref{r_seisitu}, $\{r_i^*\}$ satisfies
\begin{align}
\sum_{i\in \mathcal{G}} (r_i^*\log t)(D(\phat_i(t)\Vert S_i p)-f(n_i))_+\ge\log t\per \label{r_property}
\end{align}
Eqs.\,\eqref{cond_phat_s} and \eqref{r_property} imply that
there exists at least one $i\in \mathcal{G}$ such that
$r_i^*\log t>N_i(t)=n_i$.
This action is selected within $N$ rounds and therefore
$N_i(t')=n_i$
never holds for all $t'\ge t+N$.
Thus,
under the condition \eqref{cond_phat_s}
it holds that
\begin{align}
\sum_{t}\idx{\cA^c(t),\,\trueD(t),\,\sets(\{\phat_i(t),N_i(t)\})=\mathcal{G},
\bigcap_{i\in \mathcal{G}}\{N_i(t)=n_i\}}
\le \e^x+N\per\n
\end{align}

By using this inequality
we have
\begin{align}
\lefteqn{
\sum_{t=1}^{\infty}\idx{
\cA^c(t),\,\trueD(t)}
}\nn
&\le
\sum_{\mathcal{G}\in 2^{[N]}\setminus \emptyset}
\sum_{\{n_i\}_{i\in \mathcal{G}}\in \mathbb{N}^{|\mathcal{G}|}}
\sum_{t=1}^{\infty}\idx{
\cA^c(t),\,\trueD(t),\,\sets(\{\phat_i(t),N_i(t)\})=\mathcal{G},\,\bigcap_{i\in \mathcal{G}}\{N_i(t)=n_i\}
}\nn
&\le
\sum_{\mathcal{G}\in 2^{[N]}\setminus \emptyset}
\sum_{\{n_i\}_{i\in \mathcal{G}}\in \mathbb{N}^{|\mathcal{G}|}}
\idx{
\bigcap_{i\in \mathcal{G}}\{
D(\phat_{i,n_i}\Vert S_i p^*)\ge f(n_i)\}
}
\left(\exp\left(
\sum_{i\in \mathcal{G}}n_i(D(\phat_{i,n_i}\Vert S_i p^*)-f(n_i))
\right)+N\right)\per
\label{ac1}
\end{align}

Let \begin{equation*}
D_i=\sup_{n}\{\mathrm{ess\,sup\,} D(\phat_{i,n}\Vert S_i p^*)\}=-\log\min_{j:(S_i p^*)_j>0}(S_i p^*)_j,
\end{equation*}
where $(S_i p^*)_j$ is the $j$-th component of $S_i p^*$. Then,
\begin{align*}
\lefteqn{
\E\left[
\sum_{\{n_i\}_{i\in \mathcal{G}}\in \mathbb{N}^{|\mathcal{G}|}}
\idx{
\bigcap_{i\in \mathcal{G}}\{
D(\phat_{i,n_i}\Vert S_i p^*)\ge f(n_i)\}
}
\left(\exp\left(
\sum_{i\in \mathcal{G}}n_i(D(\phat_{i,n_i}\Vert S_i p^*)-f(n_i))
\right)+N\right)
\right]
}\nn
&\le
\sum_{\{n_i\}_{i\in \mathcal{G}}\in \mathbb{N}^{|\mathcal{G}|}}
\left(\prod_{i\in \mathcal{G}}
\int_{f(n_i)}^{D_i}
\e^{
n_i(u_i-f(n_i))
}
\rd (-P_{i,n_i}(u_i))
+N\prod_{i\in \mathcal{G}}(n_i+1)^{A}\e^{-n_i f(n_i)}
\right).
\end{align*}
The first integral is bounded as
\begin{align}
\lefteqn{
\int_{f(n_i)}^{D_i}
\e^{
n_i(u_i-f(n_i))
}
\rd (-P_i(u_i))
}\nn
&=
\left[
-\e^{n_i(u_i-f(n_i))}P_i(u_i)
\right]_{f(n_i)}^{D_i}
+
\int_{f(n_i)}^{D_i}
n_i\e^{n_i(u_i-f(n_i))}
P_i(u_i)\rd u_i
\since{integration by parts}\nn
&\le
(n_i+1)^{A}
\e^{-n_if(n_i)}
+
\int_{f(n_i)}^{D_i}
n_i(n_i+1)^{A}\e^{-n_if(n_i)}
\rd u_i
\nn
&\le
(1+n_iD_i)(n_i+1)^{A}
\e^{-n_if(n_i)}\per
\label{ac3}
\end{align}
Putting \eqref{ac1}--\eqref{ac3} together we have
\begin{align*}
\lefteqn{
\E\left[
\sum_{t=1}^{\infty}\idx{
\cA^c(t),\,\trueD(t)}
\right]
}\nn
&\le
\sum_{\mathcal{G}\in 2^{[N]}\setminus \emptyset}
\sum_{\{n_i\}_{i\in \mathcal{G}}\in \mathbb{N}^{|\mathcal{G}|}}
\left(
\prod_{i\in \mathcal{G}}(1+n_iD_i)(n_i+1)^{A}
\e^{-n_if(n_i)}
+N\prod_{i\in \mathcal{G}}(n_i+1)^{|A|}\e^{-n_i f(n_i)} 
\right)\nn
&\le
(N+1)
\sum_{\mathcal{G}\in 2^{[N]}\setminus \emptyset}
\prod_{i\in \mathcal{G}}
\sum_{n_i\in\mathbb{N}}
(1+n_iD_i)(n_i+1)^{A}
\e^{-n_if(n_i)}\nn
&=\lo(1)\per \since{by $n_i f(n_i) = \Theta(n_i^{1/2})$}\n
\end{align*}
\end{proof}

\begin{proof}[Proof of Lemma \ref{lem_marginal}]
Because $\cA(t)$, $(\cA'(t))^c$ and $\other(t)$ imply
\begin{align*}
\kyorim{p^*-\phat(t)}
&\ge
\sup_{p\in \cP_1^c}\{\kyorim{p^*-p}-\kyorim{\phat(t)-p}\}\nn
&\ge
\sup_{p\in \cP_1^c}\{\kyorim{p^*-\cP_1^c}-\kyorim{\phat(t)-p}\}\nn
&\ge
\kyorim{p^*-\cP_1^c}-\nu_{1,L}\alpha(t)\nn
&\ge
\kyorim{p^*-\cP_1^c}/2\com\n
\end{align*}
$(\cA'(t))^c,\,\empD(t)$ and $\other(t)$ imply
\begin{align*}
\max_{i}D(\phat_i(t)\Vert S_i p^*)
&\ge
2\max_{i}\kyoria{\phat_i(t)-S_i p^*}^2\nn
&\ge
2\max_{i}\left(\kyoria{S_i p^*-S_i\phat(t)}-\kyoria{\phat_i(t)-S_i\phat(t)}\right)_+^2\nn
&\ge
2\max_{i}\left(\kyoria{S_i p^*-S_i\phat(t)}-\sqrt{f(N_i(t))/4}\right)_+^2\nn
&\ge
2(\norm{S}\kyorim{p^*-\cP_1^c}/2-\sqrt{\de/2})_+^2\nn
&\ge
\normss\kyorim{p^*-\cP_1^c}^2/8\per
\since{by \eqref{assump_delta}}
\end{align*}
On the other hand, event
$\{\jt{i}{t},\,\cA^c(t),\,\allowbreak \cA'(t),\,\min_{j} N_j(t)=n\}$ occurs for at most twice since all actions are put into the list if $\{\cA^c(t),\,\cA'(t)\}$ occurred.
Thus, we have
\begin{align*}
\lefteqn{
\E\left[\sum_{n}\idx{\jt{i}{t},\,\cA(t),\,(\cA'(t))^c,\,\empD(t),\,\other(t)}
\right]
}\nn
&\le
2\E\left[
\sum_{n}\idx{\bigcup_{t}\{\cA(t),\,(\cA'(t))^c,\,\empD(t),\,\other(t),\,\min_{j} N_j(t)=n\}}
\right]\nn
&\le
2\sum_{n}
\Pr\left[\max_{j}D(\phat_j(t)\Vert S_jp^*)\ge
\normss\kyorim{p^*-\cP_1^c}^2/8,\,\bigcap_{j}\{N_j(t)\ge n\}
\right]\nn
&\le
2N\sum_{n}
(n+1)^{A}\e^{-n\normss\kyorim{p^*-\cP_1^c}^2/8} \since{by \eqref{ldpdisc}}\nn
&=
\lo(1)\per\n
\end{align*}
\end{proof}

\begin{proof}[Proof of Lemma \ref{lem_trueD}]
Recall that
\begin{align*}
\trueD^c(t)
&=
\left\{
\{D(\phat_{\ihat(t)}(t)\Vert S_{\ihat(t)} p^*)> f(N_{\ihat(t)}(t))\}
\scup
\bigcap_{j}\{D(\phat_{j}(t)\Vert S_j p^*)\le f(N_j(t))\}
\right\}\per\n
\end{align*}
Here
\begin{align}
\left\{\cA^c(t),\,\empD(t),\,\other(t),\,
\bigcap_{j}\{D(\phat_{j}(t)\Vert S_j p^*)\le f(N_j(t))\}\right\} \label{arienai}
\end{align}
cannot occur since \eqref{arienai} implies that
\begin{align}
\kyorim{\phat(t)-p^*}&=
\max_{j}\kyoria{S_j\phat(t)-S_j p^*}\normsd\nn
&\le
\max_{j}(\kyoria{S_j\phat(t)-\phat_j(t)}+\kyoria{\phat_j(t)-S_j p^*})\normsd\nn
&\le
\max_{j}\left(\sqrt{D(\phat_j(t)\Vert S_j\phat(t))/2}+
\sqrt{D(\phat_j(t)\Vert S_jp^*)/2}\right)\normsd
\nn
&\le
\sqrt{2\max_{j}f(N_j(t))}\normsd
\since{by $\empD(t)$}
\nn
&\le
2\sqrt{\de}\normsd
\since{by \eqref{prop_e1}}
\nn
&\le
\kyorim{p^*-\cP_1^c}/\sqrt{2}
\since{by \eqref{assump_delta}}
\com\n
\end{align}
which contradicts $\phat(t)\in \cP_1^c$.

On the other hand,
$\idx{\jt{i}{n},\,\ihat(t)=j,\,D(\phat_{j}(t)\Vert S_j p^*)> f(N_j(t)),\,N_j(t)=n_j}$
occurs for at most twice
since
$\ihat(t)$ is put into the list under this event.
Thus, we have
\begin{align*}
\E\left[
\sum_{t=1}^{\infty}
\idx{\jt{i}{t},\,\cA^c(t),\,\trueD^c(t),\,\empD(t),\,\other(t)}
\right]
&\le
2\sum_{j}\sum_{n=1}^{\infty}
\Pr[D(\phat_{j,\,n}\Vert S_j p^*)> f(n)]\nn
&\le
2\sum_{j}\sum_{n=1}^{\infty}
(n+1)^A\e^{-nf(n)} \since{by \eqref{ldpdisc}}\nn
&=
\lo(1)\per\n
\end{align*}
\end{proof}

\begin{proof}[Proof of Lemma \ref{lem_empD}]
$\empD^c(t)$ implies
\begin{align*}
0
&<\min_p\sum_jN_j(t) (D(\phat_j(t)\Vert S_j p)-f(N_j(t)))_+\nn
&\le\sum_jN_j(t) (D(\phat_j(t)\Vert S_j p^*)-f(N_j(t)))_+\n
\end{align*}
and therefore
\begin{align*}
\bigcup_{j}\{D(\phat_j(t)\Vert S_j p^*)\ge f(N_j(t))\}\per\n
\end{align*}
Note that
$\{\jt{i}{t},\,\empD^c(t),\,N_j(t)=n\}$ occurs for at most twice
because all actions are put into the list if $\empD^c(t)$ occurred.
Thus, we have
\begin{align*}
\lefteqn{
\E\left[\sum_{n}\idx{\jt{i}{t},\,\empD^c(t)}
\right]
}\nn
&\le
2\E\left[
\sum_{j}\sum_{n}\idx{D(\phat_{j,n}\Vert S_j p^*)\ge f(n)}
\right]\nn
&\le
2\sum_{j}\sum_n (n+1)^A \e^{-nf(n)} \since{by \eqref{ldpdisc}}\nn
&=
\lo(1)\per\n
\end{align*}
\end{proof}

\begin{proof}[Proof of Lemma \ref{lem_other}]
First, we have
\begin{align*}
\other^c(t)
&=
\bigg\{
\max_if(N_i(t))> \min\left\{2\de,\,(\norm{S}\rho_{1,L}\al(t))^2/4\right\}
\nn
&\qquad
\scup \min_i N_i(t)< \max\{c\sqrt{\log t},\,(\log \log T)^{1/3}\}\scup
2\nu_{1,L}\al(t)> \kyorim{p^*-\cP_1^c}
\bigg\}\nn
&\subset
\bigg\{
f(\min_i N_i(t))> \min\left\{2\de,\,(\norm{S}\rho_{1,L}\al(t))^2/4\right\}
\nn
&\qquad
\scup \min_i N_i(t)<c\sqrt{\log t}\scup c\sqrt{\log t}<(\log \log T)^{1/3}\scup
2\nu_{1,L}\al(t)> \kyorim{p^*-\cP_1^c}
\bigg\}\nn
&\subset
\bigg\{
\frac{b}{\sqrt{c\sqrt{\log t}}}> \min\left\{2\de,\,(\norm{S}\rho_{1,L}\al(t))^2/4\right\}
\scup\min_iN_i(t)< c\sqrt{\log t}\nn
&\qquad\scup t<\e^{\frac{(\log \log T)^{2/3}}{c}}\scup
2a/\log\log t> \kyorim{p^*-\cP_1^c}/\rho_{1,L}
\bigg\}\nn
&=
\bigg\{
t<\e^{\frac{b^4}{16\de^4c^2}}\scup
\frac{(\log t)^{1/4}}{\log \log t}<\frac{b}{a\sqrt{c}\norm{S}\rho_{1,L}}
\scup\min_iN_i(t)< c\sqrt{\log t}\nn
&\qquad\scup t<\e^{\frac{(\log \log T)^{2/3}}{c}}\scup
t<\e^{\e^{2a\rho_{1,L}/\kyorim{p^*-\cP_1^c}}}
\bigg\}\per\n
\end{align*}
From $\lim_{t\to\infty}(\log t)^{1/4}/\log \log t=\infty$
and
$\e^{\frac{(\log \log T)^{2/3}}{c}}=\so(\e^{\log \log T})=\so(\log T)$
we have
\begin{align*}
\lefteqn{
\sum_{t=1}^{T}\idx{\jt{i}{t},\,\other^c(t)}
}\nn
&=
\sum_{j}\sum_{t=1}^{T}\idx{\jt{i}{t},\,N_j(t)< c\sqrt{\log t}}
+\so(\log T)\per\n
\end{align*}
By \eqref{cond_sqrtlog},
 event $\{\jt{i}{t},\,N_j(t)< c\sqrt{\log t},\,N_j(t)=n\}$ occurs for at most twice
and therefore
\begin{align*}
\lefteqn{
\sum_{t=1}^{T}\idx{\jt{i}{t},\,\other^c(t)}
}\nn
&\le
2\sum_{j}\sum_{n=1}^T\idx{\bigcup_{t=1}^T \{n< c\sqrt{\log t}\}}
+\so(\log T)\nn
&=\so(\log T)\per\n
\end{align*}
\end{proof}

\end{document}